%% file: main_arxiv.tex
\pdfoutput=1
\documentclass{article}
\PassOptionsToPackage{numbers,sort&compress}{natbib}
\usepackage{neurips_arxiv}
\usepackage{amsthm}
\usepackage{multirow}
\usepackage{array}
\usepackage{xspace}
\usepackage{graphicx}
\usepackage{wrapfig}
\usepackage{subcaption}
\usepackage{amsfonts}
\usepackage{amsmath}
\usepackage{algorithm}
\usepackage{algpseudocode}
\usepackage{booktabs}
\usepackage[table]{xcolor}

\definecolor{ourshl}{RGB}{220, 235, 250}
\usepackage{pifont}
\usepackage{amssymb}
\usepackage{contour}
\usepackage{comment}
\contourlength{0.2pt}

\definecolor{softgreen}{RGB}{120,180,120}
\definecolor{softred}{RGB}{210,120,120}
\definecolor{ourshl}{RGB}{205,225,245}

\newcommand{\greencheck}{\contour{black}{\textcolor{softgreen}{\checkmark}}}
\newcommand{\redx}{\contour{black}{\textcolor{softred}{\ding{55}}}}

\newtheorem{theorem}{Theorem}

\newcommand{\ourstight}{\textup{\textsc{TICL}}$^3$}
\newcommand{\tabicl}{\textup{\textsc{TabICL}}\xspace}
\newcommand{\tabiclvtwo}{\textup{\textsc{TabICLv2}}\xspace}
\newcommand{\ours}{\ourstight\xspace}
\newcommand{\oursglobal}{\ourstight{} (global + residual)\xspace}
\newcommand{\oursreg}{\ourstight{} (regional anchors)\xspace}
\newcommand{\oursfull}{\emph{\textbf{T}abular In-Context \textbf{L}ocal \textbf{L}inear \textbf{L}earning}\xspace}

\newcommand{\citeauthorp}[1]{\citep{#1}}
\usepackage{hyperref}
\usepackage{hypernat}

\title{Interpretable by Design Tabular Prediction via In-Context Local Linear Learning}

\author{%
  Ran Eisenberg \\
  Faculty of Engineering \\
  Bar-Ilan University \\
  Ramat Gan, Israel \\
  \And
  Steven Berger \\
  Faculty of Engineering \\
  Bar-Ilan University \\
  Ramat Gan, Israel \\
  \And
  Ofir Lindenbaum \\
  Faculty of Engineering \\
  Bar-Ilan University \\
  Ramat Gan, Israel \\
}

\begin{document}
\maketitle

\begin{abstract}

In-context learning (ICL) lets tabular foundation models make predictions without dataset-specific training or fine-tuning. These models achieve strong predictive accuracy, but interpreting how individual features contribute to the prediction remains difficult. In biology, for example, feature contributions can help identify biomarkers or risk genes. Post-hoc interpretability methods assign feature importance scores, but these scores do not indicate what the model actually uses and can lead to explanation errors. We introduce \ours (\oursfull), an interpretable-by-design predictor that uses a tabular foundation model to guide a simple linear prediction rule. \ours first learns a shared linear model to capture dataset-level feature effects. Then, for each query, it retrieves similar training samples using representations from the foundation model and fits a second linear model to the residuals in that neighborhood. The shared model captures broad effects, while the local model corrects errors among similar samples. Their combined coefficients determine each feature contribution to the final query-specific linear score, so the same model makes and explains the prediction. Across 386 classification benchmark--dataset evaluations, \ours trails \tabiclvtwo by 1.1 percentage points and outperforms CatBoost in 266 evaluations. It improves average support F1 over MAPLE by 4.7 points and achieves the best performance on both real-data selected-feature utility metrics.
\end{abstract}

\section{Introduction}

Making predictions on tabular data remains a challenge in machine learning, where interpretability of individual predictions is often as important as predictive accuracy. Classical algorithms, including sparse linear models, generalized additive models (GAMs), and explainable boosting machines (EBMs)~\citeauthorp{nori2019interpretml}, provide explicit feature contributions, although their structural constraints may limit predictive performance. Recent tabular foundation models~\citeauthorp{hollmann2023tabpfn, qu2025tabicl} adopt a different paradigm by leveraging labeled examples as context to predict new samples without requiring dataset-specific parameter updates. Related in-context regression approaches have also been extended to reinforcement learning~\citeauthorp{schiff2026icrrl}. While these models show strong predictive performance, multiple attention layers complicate attributing individual input feature contributions to each prediction.

Post-hoc interpretability methods such as LIME~\citeauthorp{ribeiro2016lime} and SHAP~\citeauthorp{lundberg2017shap} explain predictors after training. Other approaches include gradient-based attribution methods~\citeauthorp{sundararajan2017integrated,shrikumar2017learning}, FastSHAP~\citeauthorp{jethani2022fastshap}, and learned local surrogate models~\citeauthorp{piaggesi2025illume}. While these methods can provide useful feature importance scores, they do not use these scores to make predictions and may not fully reflect how the model uses each feature.

\ours is introduced as an interpretable-by-design predictor that leverages a tabular foundation model to infer a simple prediction rule. First, we train a shared ElasticNet~\citep{zou2005elastic} to learn a dataset-wide rule based on the foundation model’s predictions. For each query, the foundation model's embedding identifies similar training samples, referred to as a neighborhood. We then fit a second ElasticNet to the residuals between the foundation model’s predictions and the shared rule within this neighborhood. Integrating this local correction with the shared model produces a query-specific linear rule whose coefficients directly determine the final prediction. Figure~\ref{fig:method_overview} summarizes this procedure. In \ours, the foundation model's learned representation guides sample selection for local correction, while the shared and local coefficients specify each feature's contribution to the prediction.

The central idea of \ours is to use the foundation model’s learned representation to inform local correction fitting, with explicit coefficients specifying the prediction process. Table~\ref{tab:method_positioning} compares \ours with existing methods. In \ours, we derive feature contributions directly from the linear model used to compute each class score, rather than from a separate explanation technique. When combined with a baseline, these contributions sum precisely to the class score. The model further employs representations learned through tabular in-context learning (ICL) to select relevant training samples. However, these contributions do not elucidate how individual features influence sample selection or the determination of fitted coefficients. Figure~\ref{fig:selling-point} compares prediction accuracy, explanation quality, and computational cost.

\begin{figure}[t]
  \centering
  \includegraphics[width=\textwidth]{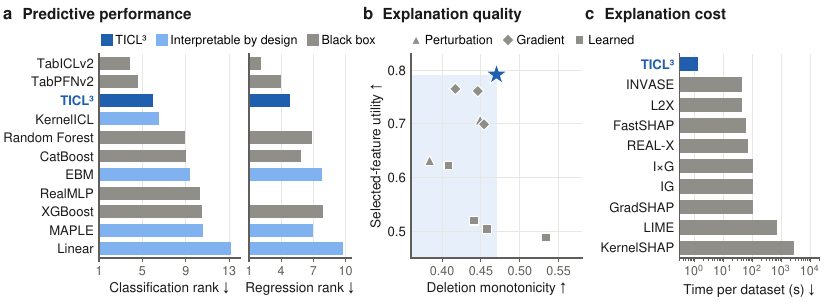}
  \caption{\textbf{Summary of results.} \textbf{(a)} Average per-dataset rank in classification and regression over all methods in Tables~\ref{tab:main} and~\ref{tab:regression}; a subset is shown. \textbf{(b)} Deletion monotonicity and selected-feature utility; the latter averages Gen-SVM and Gen-kmeans, reported separately in Table~\ref{tab:faithfulness}. \textbf{(c)} Per-dataset cost for 50 predictions and explanations, including method-specific explainer construction. All methods explain the same \oursglobal predictor, so its setup is excluded for every method. Online latency is similar for \oursglobal and the learned explainers (about 27~ms per query); the difference comes from explainer construction. LIME and KernelSHAP are extrapolated from the fixed 10-query-per-dataset subset.}
  \label{fig:selling-point}
\end{figure}

\begin{table}[t]
\centering
\scriptsize
\caption{
Conceptual positioning of \ours. \emph{Interpretable by design}: the explanation is part of the deployed predictor; \emph{explicit additive prediction}: prediction sums inspectable feature terms mapped to input columns; \emph{query-specific deployed rule}: an explicit feature-level rule is fitted or selected per test point (sample weights or post-hoc explanations alone do not qualify); \emph{ICL-guided locality}: an ICL representation defines locality. KernelICL uses ICL-derived query weights but no query-specific feature-level rule. ShapPFN sums per-feature contributions computed by its transformer but deploys no explicit feature-level rule; ExplainerPFN estimates attributions without access to a predictor.
}
\label{tab:method_positioning}
\resizebox{\linewidth}{!}{%
\begin{tabular}{p{0.27\columnwidth}cccc}
\toprule
\textbf{Method}
& \textbf{Interpretable by design}
& \textbf{Explicit additive prediction}
& \textbf{Query-specific deployed rule}
& \textbf{ICL-guided locality} \\
\midrule
\tabiclvtwo
& \redx
& \redx
& \redx
& \redx \\

SHAP/LIME on \tabiclvtwo
& \redx
& \redx
& \redx
& \redx \\

ExplainerPFN
& \redx
& \redx
& \redx
& \redx \\

EBM/NAM
& \greencheck
& \greencheck
& \redx
& \redx \\

MAPLE
& \greencheck
& \greencheck
& \greencheck
& \redx \\

LLSPIN
& \greencheck
& \greencheck
& \greencheck
& \redx \\

KernelICL
& \greencheck
& \redx
& \redx
& \greencheck \\

ShapPFN
& \greencheck
& \greencheck
& \redx
& \redx \\

\rowcolor{ourshl}
\oursglobal
& \greencheck
& \greencheck
& \greencheck
& \greencheck \\
\bottomrule
\end{tabular}
}

\vspace{2pt}
\begin{minipage}{0.96\columnwidth}
\scriptsize
\emph{References:} 
\tabiclvtwo~\citeauthorp{qu2026tabiclv2};
SHAP/LIME~\citeauthorp{lundberg2017shap,ribeiro2016lime};
ExplainerPFN~\citeauthorp{fonseca2026explainerpfn};
EBM/NAM~\citeauthorp{nori2019interpretml,agarwal2021nam};
MAPLE~\citeauthorp{plumb2018maple};
LLSPIN~\citeauthorp{yang2022locally};
KernelICL~\citeauthorp{miftachov2026interpretable};
ShapPFN~\citeauthorp{sena2026shappfn}.
\end{minipage}
\end{table}

For each query, we form a single linear model that is composed of the shared and local models. The coefficients, together with the processed feature values, determine each feature’s contribution to the prediction score for a class, and we map these contributions back to the original columns of the table. We use an ElasticNet regularizer to encourage a sparse coefficient vector. The foundation model does not make the final prediction; it provides targets to fit the shared prediction rule and selects the rows to use for each query. The final prediction is made by the explicit prediction rule described above. The error in this local approximation can be bounded in terms of (i) smoothness of the residual, (ii) context shift, (iii) the radius of the neighborhood in the representation space, (iv) distance of the query to the weighted feature centroid, and (v) the amount of regularization.

\begin{figure*}[t]
  \centering
  \includegraphics[width=\textwidth]{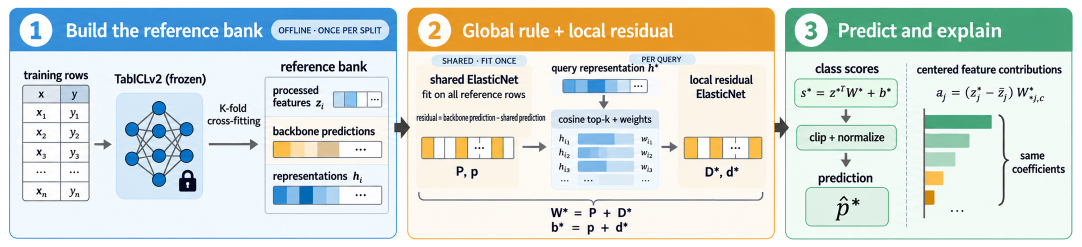}
  \caption{\textbf{Overview of \oursglobal.} Offline we use a frozen \tabiclvtwo{} backbone with cross-fitting to build a reference bank of processed features, class probability targets and learned representations. We fit a shared ElasticNet model to the bank targets. For each query, we then use its representation to select similar reference rows and fit a second ElasticNet model to the shared model's remaining errors on those rows. We add the shared model and the local correction to give a single linear model for the query. We clip and normalize the class scores from this to make the final prediction. We also use the model coefficients and the centered feature values to compute the contribution of each feature to the class scores.}
  \label{fig:method_overview}
\end{figure*}

\section{Background and Related Work}

We group related work into three categories: predictive models, post-hoc explanation methods, and interpretable-by-design models.

\paragraph{Predictive models.}
Tabular ICL models such as TabPFN~\citeauthorp{hollmann2023tabpfn}, TabPFNv2~\citeauthorp{hollmann2025tabpfnv2}, and \tabiclvtwo~\citeauthorp{qu2026tabiclv2} demonstrate strong predictive performance, as do task-specific neural models like TabM and TabR~\citeauthorp{gorishniy2025tabm,gorishniy2024tabr}. Some approaches combine neural and tree-based encoders to learn tabular representations~\citeauthorp{naor2025hybrid}. While these models learn effective representations and prediction functions, they do not provide explicit feature contributions that directly determine predictions.

\paragraph{Post-hoc explanation methods.}
LIME~\citeauthorp{ribeiro2016lime}, SHAP~\citeauthorp{lundberg2017shap}, gradient attributions~\citeauthorp{sundararajan2017integrated,shrikumar2017learning}, and learned explainers clarify the outputs of a fixed predictor. ExplainerPFN~\citeauthorp{fonseca2026explainerpfn} uses a pretrained tabular foundation model to estimate feature attributions for a new dataset without access to the explained predictor. These methods offer feature attributions or local approximations, but their explanations do not determine the final prediction.

\paragraph{Interpretable-by-design models.}
Lasso~\citeauthorp{tibshirani1996lasso}, EBMs~\citeauthorp{nori2019interpretml}, and NAMs~\citeauthorp{agarwal2021nam} use explicit feature contributions within a shared model. LOESS~\citeauthorp{cleveland1988loess} and MAPLE~\citeauthorp{plumb2018maple} fit local linear rules, defining neighborhoods or sample weights by input distances or tree-based representations. LSPIN and its locally linear variant LLSPIN~\citeauthorp{yang2022locally} train a gating network that selects a sparse feature subset for each sample; with a linear prediction network, LLSPIN deploys a sample-specific sparse linear rule. ShapPFN~\citeauthorp{sena2026shappfn} makes a tabular foundation model's prediction additive over per-feature contributions, although each contribution is computed by the transformer rather than by an explicit feature-level rule. Context-dependent feature-selection methods allow selected variables to vary with the observed context~\citeauthorp{sristi2024contextual}. However, these methods do not leverage context-conditioned representations from a pretrained tabular ICL model to fit a query-specific prediction rule.

\ours uses representations from a tabular ICL model to select a neighborhood for each query and fits a local linear correction to the residuals of a shared linear predictor. The combined coefficients directly determine the prediction. For more details on related work, see Appendix~\ref{app:related_work}.

\section{Method}
\label{sec:method}

For a new test query, \ours constructs an adaptive linear prediction rule using representations from a tabular foundation model. The representations allow \ours to identify the training samples that are most relevant to the test query. The coefficients, together with the processed feature values and intercept, determine the class scores and feature contributions. We consider two variants of \ours, \oursglobal and \oursreg. In \oursglobal, we first fit a shared linear prediction model to capture the global effects of features. Then, for each test query, we use the representations to find training samples similar to the query and fit a second linear prediction model to the shared model's residuals on these samples. The final prediction rule for the test query consists of the sum of the shared linear model and the query-specific linear model. In \oursreg, we avoid fitting a new prediction model for each test query by fitting regional prediction models upfront and combining them based on their similarity to the query. See \autoref{par:regional_anchor_variant} for details. We use the terms ``linear model'' and ``linear predictor'' to refer to linear models with an intercept that is not regularized.

\paragraph{Backbone outputs and reference bank.}

Suppose we have a labeled training set $\{(\mathbf{x}_i, y_i)\}_{i=1}^{N}$, where $\mathbf{x}_i \in \mathbb{R}^{D_{\rm raw}}$ is the $i$-th sample's feature vector and $y_i \in \{1, \ldots, C\}$ is its class. We apply a deterministic mapping $\phi: \mathbb{R}^{D_{\rm raw}} \rightarrow \mathbb{R}^D$ to obtain the processed features $\mathbf{z}_i = \phi(\mathbf{x}_i)$. In our experiments, $\phi$ encodes categorical features using ordinal encoding, imputes missing values with the corresponding training-set means, and applies scalar transformations fitted on the training set. We do not expand features, so each retained coordinate of $\mathbf{z}_i$ corresponds to a single coordinate of $\mathbf{x}_i$. See Appendix~\ref{app:preprocessing_mapping} for details about the transformations and how we map feature contributions and recovered supports back to the original columns.

For each sample, our tabular foundation model (backbone) outputs a predicted class-probability vector $\widetilde{\mathbf{y}}_i \in \Delta^{C-1}$ and a representation $\mathbf{h}_i \in \mathbb{R}^{d_h}$, where $\Delta^{C-1} := \{\mathbf{u} \in \mathbb{R}_{\geq 0}^{C} : \mathbf{1}^{\top}\mathbf{u} = 1\}$ is the probability simplex and $\mathbf{1}$ is the all-ones vector. All vectors are column vectors. We use $\|\cdot\|_2$ to denote the $\ell_2$ norm of a vector.

We use stratified cross-fitting to generate the backbone outputs so that a sample's class label is not included in the labeled context used to compute its class-probability vector and representation. Specifically, we use the labeled context from all other folds when generating the backbone outputs for samples in a given fold. We then aggregate the outputs from all folds to construct a reference bank $\{(\mathbf{z}_i, \widetilde{\mathbf{y}}_i, \mathbf{h}_i)\}_{i=1}^{N}$, in which each training sample's prediction and representation are generated without its own label in the context. See Appendix~\ref{app:reference_bank} for more details.

Given a new test input $\mathbf{x}^*$, we transform its features using the same mapping $\mathbf{z}^* = \phi(\mathbf{x}^*)$ and generate its representation $\mathbf{h}^*$ using the tabular foundation model with the entire training set as its labeled context. We fit the mapping $\phi$ using the entire training set and use this fitted mapping for all samples. We only cross-fit the backbone predictions and representations for the training samples.

\subsection{Context-Conditioned Neighborhood}

We use the backbone representations to identify training rows most similar to the test query. Cosine similarity is computed as $s_i(\mathbf{x}^*)=\mathbf h_i^\top\mathbf h^*/(\|\mathbf h_i\|_2\|\mathbf h^*\|_2)$, and the $k$ most similar rows are selected. For a neighborhood size $k\in\mathbb N_{+}$, $\mathcal N^*\subseteq\{1,\ldots,N\}$ denotes the indices of the $k'=\min(k,N)$ largest similarities. This approach maintains a fixed number of rows for each local fit, rather than varying the count based on a similarity threshold. Greater weight is assigned to rows that are more similar to the query. For a temperature parameter $\tau>0$, we define
\begin{equation}
\widetilde w_i^*=
\frac{\exp(s_i(\mathbf{x}^*)/\tau)}
{\sum_{j\in\mathcal N^*}\exp(s_j(\mathbf{x}^*)/\tau)},
\qquad i\in\mathcal N^*.
\label{eq:residual_neighbor_weights}
\end{equation}
The weights are strictly positive and sum to one. Smaller $\tau$ assigns greater weight to the closest neighbors, whereas larger values result in more uniform weighting.

\subsection{Shared Global Predictor}
We fit a global predictor to capture the part of the backbone's predictions that can be explained by a single linear model across the entire dataset. Given the coefficient matrix $\mathbf A \in \mathbb{R}^{D \times C}$, we use the ElasticNet regularizer $\Omega_{\alpha, \rho}(\mathbf A) = \alpha \rho \|\mathbf A\|_1 + \alpha (1 - \rho) \|\mathbf A\|_F^2 / 2$ with regularization strength $\alpha > 0$ and mixing parameter $\rho \in [0, 1]$. Here, $\|\mathbf A\|_1$ denotes the sum of the absolute values of its entries. The $\ell_1$ term encourages sparsity, while the $\ell_2$ term penalizes large coefficients. Specifically, we learn the dataset-wide model $(\mathbf P, \mathbf p)$ by solving
\begin{align*}
(\mathbf P, \mathbf p)
=
\arg\min_{\mathbf A \in \mathbb{R}^{D \times C},\, \mathbf a \in \mathbb{R}^C}
\frac{1}{2N}\sum_{i=1}^N
\|\mathbf A^\top \mathbf z_i + \mathbf a - \widetilde{\mathbf y}_i\|_2^2
+ \Omega_{\alpha, \rho}(\mathbf A),
\end{align*}
where $\mathbf P \in \mathbb{R}^{D \times C}$ is the coefficient matrix, $\mathbf p \in \mathbb{R}^C$ is the unregularized intercept vector, and $\mathbf P^\top \mathbf z + \mathbf p$ is the prediction from the dataset-wide linear model for the processed feature vector $\mathbf z$. This model is fitted once per training split and shared across all test queries.

\subsection{Query-Specific Residual Correction}

The dataset-wide model may be unable to capture all of the variation in the backbone's predictions. To account for this, we train a second model to predict the residual $\mathbf r_i = \widetilde{\mathbf y}_i - (\mathbf P^\top \mathbf z_i + \mathbf p)$ for the reference rows. For a test query $\mathbf x^*$, the query-specific residual correction $(\mathbf D^*, \mathbf d^*)$ is learned by solving
\begin{align}
(\mathbf D^*, \mathbf d^*)
=
\arg\min_{\mathbf D \in \mathbb{R}^{D \times C},\, \mathbf d \in \mathbb{R}^C}
\frac{1}{2}\sum_{i \in \mathcal{N}^*}
\widetilde w_i^*
\|\mathbf D^\top \mathbf z_i + \mathbf d - \mathbf r_i\|_2^2
+ \Omega_{\alpha, \rho}(\mathbf D),
\label{eq:local_residual_elasticnet}
\end{align}
where $\mathbf d^*$ is the unregularized intercept vector. This model is fitted separately for each test query. We use $\mathbf W^* = \mathbf P + \mathbf D^*$ and $\mathbf b^* = \mathbf p + \mathbf d^*$ to define the final linear prediction rule, where $\mathbf P$ contains the shared coefficients and $\mathbf D^*$ contains the query-specific correction learned from the query's neighborhood. This rule outputs a vector with one score per class, given by $(\mathbf W^*)^\top \mathbf z^* + \mathbf b^*$.

We next analyze the error in approximating the backbone's predictions by fitting the residuals left by the dataset-wide model. Let $g: \mathbb{R}^D \rightarrow \mathbb{R}^C$ be the dataset-wide linear model, where $g(\mathbf z) = \mathbf P^\top \mathbf z + \mathbf p$. We let $\boldsymbol\eta(\mathbf x; \mathcal{C})$ and $\mathbf h(\mathbf x; \mathcal{C})$ denote the backbone's class-probability vector and learned query representation, respectively, given a context of labeled examples $\mathcal{C}$. We assume that the representations are nonzero. For each reference row $i$, we let $\boldsymbol\eta_i = \boldsymbol\eta(\mathbf x_i; \mathcal{C}_i) = \widetilde{\mathbf y}_i$ and $\mathbf h_i = \mathbf h(\mathbf x_i; \mathcal{C}_i)$, where $\mathcal{C}_i$ consists of the training examples outside the fold containing $\mathbf x_i$. For a test query $\mathbf x^*$, we let $\boldsymbol\eta^* = \boldsymbol\eta(\mathbf x^*; \mathcal{C}^*)$ and $\mathbf h^* = \mathbf h(\mathbf x^*; \mathcal{C}^*)$, where $\mathcal{C}^*$ consists of all training examples. Let $\mathbf r_i = \boldsymbol\eta_i - g(\mathbf z_i)$ and $\mathbf r^* = \boldsymbol\eta^* - g(\mathbf z^*)$. We normalize the representations as $\widehat{\mathbf h}_i = \mathbf h_i/\|\mathbf h_i\|_2$ and $\widehat{\mathbf h}^* = \mathbf h^*/\|\mathbf h^*\|_2$. We use $\Delta_k(\mathbf x^*)=\max_{i\in\mathcal N^*}\|\widehat{\mathbf h}_i-\widehat{\mathbf h}^*\|_2$ for the representation-space radius of the retrieved neighborhood and $\overline{\mathbf z}_k=\sum_{i\in\mathcal N^*}\widetilde w_i^*\mathbf z_i$ for its weighted centroid in processed-feature space.

\begin{theorem}[Residual approximation bound]
\label{thm:residual_approx}
Assume $\alpha(1-\rho)>0$ and that there are constants $L_r\geq0$ and $\epsilon_{\rm ctx}\geq0$ such that each retrieved residual is close to the query residual:
\[
\|\mathbf r_i-\mathbf r^*\|_2
\leq
L_r
\|\widehat{\mathbf h}_i-\widehat{\mathbf h}^*\|_2
+\epsilon_{\rm ctx},
\qquad i\in\mathcal N^*.
\]
Here $L_r$ measures residual smoothness in representation space, and $\epsilon_{\rm ctx}$ absorbs the shift between outputs produced under the out-of-fold contexts $\mathcal C_i$ and under the full context $\mathcal C^*$. Then the linear prediction before clipping and normalization,
\[
\widehat{\boldsymbol\eta}_{\rm lin}^*
=
g(\mathbf z^*)+
(\mathbf D^*)^\top\mathbf z^*+\mathbf d^*,
\]
satisfies
\begin{equation}
\|\boldsymbol\eta^*-\widehat{\boldsymbol\eta}_{\rm lin}^*\|_2
\leq
\bigl(L_r\Delta_k(\mathbf x^*)+\epsilon_{\rm ctx}\bigr)
\left(
1+
\frac{\|\mathbf z^*-\overline{\mathbf z}_k\|_2}
{\sqrt{\alpha(1-\rho)}}
\right).
\label{eq:residual_approx_bound}
\end{equation}
\end{theorem}

The proof can be found in Appendix~\ref{app:residual_approx_proof}. When $\epsilon_{\rm ctx}$ is 0, the condition says that the Euclidean norm of the difference between the residual of the query and of each retrieved row is bounded by the distance between their normalized representations, multiplied by $L_r$. $\epsilon_{\rm ctx}$ allows this residual difference to be larger, to account for differences arising from the labeled contexts used to obtain the backbone outputs.
When $\mathbf D^*$ is all zeroes, the correction is the weighted mean of the neighborhood residuals $\overline{\mathbf r}_k$, and the error is bounded by $L_r \Delta_k(\mathbf x^*) + \epsilon_{\rm ctx}$. The theorem doesn't show that the slope will make the correction better, but it bounds the additional error that the slope may introduce through the ElasticNet regularization. The fitted slope, combined with the shared coefficients, can be used to compute feature contributions.
We can set $g \equiv \mathbf 0$ to get a similar bound for approximating the backbone outputs directly, under the corresponding smoothness assumption on those outputs. Using the shared model might make the term $L_r \Delta_k(\mathbf x^*) + \epsilon_{\rm ctx}$ smaller by reducing $L_r$ or $\epsilon_{\rm ctx}$, but it doesn't necessarily do so, because the shared model is linear with respect to the processed features, while $L_r$ describes variation with respect to distance in the representation space. Also, although we use cross-fitting so that the label for the row is not used in its context, we don't actually bound $\epsilon_{\rm ctx}$. Finally, the bound is only for the linear scores; next we describe how to clip and normalize them.

\paragraph{Prediction and feature contributions.}

Given a test query, we first obtain class scores for the query, given by $\widehat{\mathbf{z}}_{\rm cls}^* = (\mathbf{W}^*)^\top \mathbf{z}^* + \mathbf{b}^* = \widehat{\boldsymbol{\eta}}_{\rm lin}^*$, and then we clip the scores to $[0, 1]$ and normalize them to sum to one to obtain predicted probabilities, $\widehat{\mathbf{p}}^* = \operatorname{clip}_{[0, 1]}(\widehat{\mathbf{z}}_{\rm cls}^*) / (\mathbf{1}^\top \operatorname{clip}_{[0, 1]}(\widehat{\mathbf{z}}_{\rm cls}^*))$; if the normalization constant is zero, we set $\widehat{\mathbf{p}}^* = \mathbf{1}/C$.

We explain the model using the same set of coefficients: for each processed feature $j$ and class $c$, the feature's contribution to the class score is defined as $a_{j,c}(\mathbf{x}^*) = (z_j^* - \overline{z}_j) W^*_{j,c}$, where $\overline{\mathbf{z}} \triangleq N^{-1} \sum_{i=1}^N \mathbf{z}_i$ is the average processed feature vector computed on the training set. This leads to the decomposition: $\widehat{z}_{{\rm cls}, c}^* = \underbrace{b_c^* + (\mathbf{W}_c^*)^\top \overline{\mathbf{z}}}_{\text{baseline}} + \sum_{j=1}^D W^*_{j, c} (z_j^* - \overline{z}_j)$, where the baseline is query-dependent.

The feature contributions therefore describe how much each processed feature contributes to the fitted linear class score for the query, relative to a query-dependent baseline; they do not decompose the normalized probability, because clipping and normalization are nonlinear, nor do they describe how the feature affected the query representation, which neighbors were retrieved, or the coefficients that were fit. (In particular, even if a feature is displayed to have zero coefficient, it does not mean that the entire query-adaptive prediction pipeline is independent of the feature value.) Unless the class is fixed by the evaluation protocol, we report feature contributions for the predicted class. We also map the contributions back to the raw, semantic feature columns.

Optimization. We solve the two ElasticNet problems with weighted centering to allow for an unregularized intercept, optimizing each coefficient matrix jointly across classes with FISTA~\citep{beck2009fista} with backtracking line search and checking the KKT residual, for at most 300 iterations with tolerance $10^{-4}$. When solving for the coefficients of the shared model, all rows of the reference set are weighted equally, and when solving for the query-specific correction, we weight the rows according to equation~\ref{eq:residual_neighbor_weights}.

Implementation details. For the primary configuration (global residual), we use one \tabiclvtwo estimator for test-query predictions, using float32 precision (not automatic mixed precision). We use the class probabilities directly as the target when fitting the shared model (there is no label smoothing). When computing the representation, we use a single view at the input to the first ICL decoder block; when retrieving neighbors, we use a cached, cross-fitted reference bank (see appendix~\ref{app:reference_bank}) and exact top-$k$ retrieval.

Our primary configuration uses $k = 400$ and $\tau = 0.1$; $\alpha = 0.01$ and $\rho = 0.25$ at both stages (implying $\ell_1$ coefficient $\alpha \rho = 0.0025$, ridge weight $\alpha(1-\rho) = 0.0075$, and the coefficient on $\|\mathbf{A}\|_F^2$ is $0.00375$); when there are fewer than $k = 400$ training rows, we use all of them but still weight them according to their representations. We use squared loss.

\phantomsection
\label{par:regional_anchor_variant}
\paragraph{Regional anchor variant.}
In \oursreg, we fit local linear models in advance rather than fitting a new model for each query. We cluster the reference-bank representations and select the training sample closest to each cluster center as an anchor. For each anchor, we fit an ElasticNet using nearby reference samples and their backbone probability targets. Unlike \oursglobal, these models are fitted directly to the backbone predictions rather than to residuals of a shared model. For a new query, we use its backbone representation to weight the coefficients and intercepts of the pre-fitted models. Their weighted combination gives a query-specific linear rule for prediction and feature contributions, without fitting a new model at prediction time (Appendix~\ref{app:regional_anchors}).

\section{Experiments}
\label{sec:experiments}

We evaluate three qualities of \ours: whether the deployed global-plus-local predictor remains a competitive tabular classifier; whether its combined coefficients rank data-generating features when the generating support is known; and the deletion monotonicity, selected-feature utility, and stability of its centered contributions on real-world datasets. Predictive performance is evaluated on 386 benchmark--dataset pairs across four classification benchmarks: CC-18~\citeauthorp{bischl2021openml} (69 datasets), TabArena-Lite~\citeauthorp{erickson2025tabarena} (38), TALENT~\citeauthorp{ye2024closer} (181), and BeyondArena~\citeauthorp{purucker2026beyond} (98). We evaluate feature-selection quality on controlled synthetic tasks with designated generating supports. Real-data interpretability is evaluated on a fixed CC-18 subset, with regional explanation diagnostics confined to Appendix~\ref{app:regional_explanation_results}.

\subsection{Experimental Setup}
\label{sec:experimental_setup}

We report mean accuracy on benchmark-specific splits. CC-18 and TabArena-Lite use stratified 80/20 splits generated independently for seeds 0, 1, and 2. TALENT keeps its provided train/validation/test partition, with training and validation sets combined, and uses seed 42 for capped subsampling. BeyondArena keeps its provided splits at split/repeat positions 0, 1, and 2. In all suites, we cap the training and test sets at 1600 and 400 rows. \ours uses one backbone estimator, a five-fold cross-fitted reference bank, and the primary hyperparameters of Section~\ref{sec:method}; Table~\ref{tab:runtime} uses separate runtime operating points (Appendix~\ref{app:prediction_runtime}).

We compare against \tabiclvtwo, KernelICL~\citeauthorp{miftachov2026interpretable},
TabPFNv2~\citeauthorp{hollmann2025tabpfnv2}, modern neural baselines
~\citeauthorp{gorishniy2024tabr,holzmueller2024realmlp,ye2025modernnca,gorishniy2021revisiting},
interpretable additive models~\citeauthorp{nori2019interpretml,agarwal2021nam}, and classical tree/linear baselines. Appendix~\ref{app:predictive_baseline_configs} gives versions, fixed settings, validation use, and suite-specific ensemble sizes.

\subsection{Predictive Performance}

Table~\ref{tab:main} compares the completed primary configuration with foundation-model, interpretable, neural, and classical baselines. \oursglobal is within $0.3$--$1.6$ percentage points of \tabiclvtwo across the four suites and exceeds the strongest tree baseline by $0.7$--$1.8$ points. \oursreg trades a small amount of accuracy for faster prediction, as quantified in Table~\ref{tab:runtime}. Its explanations are less informative, however: its anchor models use a ten-times larger $\ell_1$ penalty and no shared component, and the blended predicted-class coefficients are all zero for 73.3\% of synthetic and 43.8\% of real-data queries, whose predictions then vary only through anchor routing and blended intercepts (Appendix~\ref{app:regional_explanation_results}). We therefore treat it as a computational variant rather than a coequal explanation method.

\begin{table}[t]
  \centering
  \caption{Predictive performance on CC-18, TabArena-Lite, TALENT, and BeyondArena under Section~\ref{sec:experimental_setup}. In that order, \tabiclvtwo{} uses 8/4/4/4 estimators and TabPFNv2 uses 8/8/4/8; both \ours variants use one. The \ours rows use a five-fold cross-fitted bank and the configurations in Section~\ref{sec:method} and Appendix~\ref{app:regional_anchors}. \oursreg is grouped by its predictor form, but its predicted-class explanation is all zero for 73.3\% of synthetic and 43.8\% of real-data queries (Appendix~\ref{app:regional_explanation_results}). The $\pm$ values describe variation across dataset means.}
  \label{tab:main}
  \scriptsize
  \setlength{\tabcolsep}{2pt}
  \resizebox{\columnwidth}{!}{%
  \begin{tabular}{llcccc}
    \toprule
    \textbf{Family}
      & \textbf{Method}
      & \textbf{CC-18}$\uparrow$
      & \textbf{TabArena-Lite}$\uparrow$
      & \textbf{TALENT}$\uparrow$
      & \textbf{BeyondArena}$\uparrow$ \\
    \midrule
    \multirow{2}{*}{\shortstack[l]{Tabular foundation\\models}}
      & \tabiclvtwo
      & $\mathbf{0.872 \pm 0.154}$
      & $\mathbf{0.880 \pm 0.087}$
      & $\mathbf{0.830 \pm 0.148}$
      & $\mathbf{0.839 \pm 0.132}$ \\
      & TabPFNv2
      & $\underline{0.865 \pm 0.157}$
      & $\mathbf{0.880 \pm 0.086}$
      & $\underline{0.825 \pm 0.149}$
      & $\underline{0.835 \pm 0.137}$ \\
    \midrule
    \multirow{5}{*}{\shortstack[l]{Interpretable\\by design}}
      & \cellcolor{ourshl}\oursglobal
      & \cellcolor{ourshl}$0.856 \pm 0.161$
      & \cellcolor{ourshl}$\underline{0.877 \pm 0.089}$
      & \cellcolor{ourshl}$0.817 \pm 0.160$
      & \cellcolor{ourshl}$0.833 \pm 0.134$ \\
      & \cellcolor{ourshl}\oursreg
      & \cellcolor{ourshl}$0.855 \pm 0.162$
      & \cellcolor{ourshl}$0.874 \pm 0.087$
      & \cellcolor{ourshl}$0.814 \pm 0.164$
      & \cellcolor{ourshl}$0.834 \pm 0.136$ \\
      & MAPLE
      & $0.830 \pm 0.164$
      & $0.862 \pm 0.091$
      & $0.790 \pm 0.166$
      & $0.820 \pm 0.137$ \\
      & EBM
      & $0.840 \pm 0.159$
      & $0.865 \pm 0.093$
      & $0.779 \pm 0.186$
      & $0.825 \pm 0.136$ \\
      & NAM
      & $0.827 \pm 0.165$
      & $0.859 \pm 0.094$
      & $0.774 \pm 0.182$
      & $0.808 \pm 0.144$ \\
    \midrule
    \multirow{2}{*}{\shortstack[l]{Kernel ICL\\surrogates}}
      & KernelICL (kNN, $k{=}20$)
      & $0.825 \pm 0.206$
      & $0.872 \pm 0.094$
      & $0.803 \pm 0.177$
      & $0.823 \pm 0.150$ \\
      & KernelICL (Gaussian, scale $4$)
      & $0.826 \pm 0.205$
      & $0.875 \pm 0.091$
      & $0.806 \pm 0.173$
      & $0.825 \pm 0.152$ \\
    \midrule
    \multirow{4}{*}{\shortstack[l]{Deep tabular\\models}}
      & TabR
      & $0.840 \pm 0.165$
      & $0.855 \pm 0.095$
      & $0.797 \pm 0.160$
      & $0.798 \pm 0.154$ \\
      & RealMLP
      & $0.846 \pm 0.164$
      & $0.862 \pm 0.094$
      & $0.790 \pm 0.174$
      & $0.816 \pm 0.144$ \\
      & ModernNCA
      & $0.843 \pm 0.165$
      & $0.862 \pm 0.094$
      & $0.798 \pm 0.164$
      & $0.819 \pm 0.139$ \\
      & FT-Transformer
      & $0.836 \pm 0.167$
      & $0.860 \pm 0.092$
      & $0.783 \pm 0.175$
      & $0.813 \pm 0.145$ \\
    \midrule
    \multirow{5}{*}{\shortstack[l]{Classical\\baselines}}
      & CatBoost
      & $0.849 \pm 0.161$
      & $0.866 \pm 0.092$
      & $0.799 \pm 0.166$
      & $0.824 \pm 0.141$ \\
      & Random Forest
      & $0.842 \pm 0.164$
      & $0.867 \pm 0.091$
      & $0.797 \pm 0.162$
      & $0.824 \pm 0.139$ \\
      & LightGBM
      & $0.842 \pm 0.165$
      & $0.862 \pm 0.096$
      & $0.795 \pm 0.165$
      & $0.817 \pm 0.147$ \\
      & XGBoost
      & $0.842 \pm 0.162$
      & $0.863 \pm 0.093$
      & $0.794 \pm 0.165$
      & $0.816 \pm 0.148$ \\
      & Logistic Regression
      & $0.764 \pm 0.199$
      & $0.806 \pm 0.122$
      & $0.738 \pm 0.205$
      & $0.780 \pm 0.157$ \\
    \bottomrule
  \end{tabular}%
  }
\end{table}

\begin{table}[t]
  \centering
  \caption{CC-18 accuracy and mean fit/prediction time over 69 datasets and three seeds. The default \tabiclvtwo{} row uses eight estimators; the matched row uses one. Configuration and timing details are in Appendix~\ref{app:prediction_runtime}.}
  \label{tab:runtime}
  \small
  \setlength{\tabcolsep}{10pt}
  \begin{tabular}{lccc}
    \toprule
    \textbf{Method} & \textbf{Mean Acc.}$\uparrow$ & \textbf{Fit (s)}$\downarrow$ & \textbf{Predict (s)}$\downarrow$ \\
    \midrule
    \tabiclvtwo{} (8 estimators) & $\mathbf{0.8735}$ & $\underline{0.950}$ & $5.462$ \\
    \tabiclvtwo{} (1 estimator) & $\underline{0.8710}$ & $\mathbf{0.485}$ & $\mathbf{0.745}$ \\
    TabPFNv2 & $0.8641$ & $3.139$ & $9.660$ \\
    \rowcolor{ourshl}
    \oursglobal & $0.8557$ & $5.515$ & $6.849$ \\
    \rowcolor{ourshl}
    \ourstight{}-Fast (global + residual) & $0.8609$ & $5.842$ & $1.495$ \\
    \rowcolor{ourshl}
    \oursreg & $0.8555$ & $6.736$ & $\underline{0.850}$ \\
    \bottomrule
  \end{tabular}
\end{table}

In Table~\ref{tab:runtime}, the fast configuration reduces global-residual prediction time from $6.849$ to $1.495$~s; regional anchors require $0.850$~s. Their speed advantage is relative to the eight-estimator backbone: all TICL variants predict more slowly than the matched one-estimator backbone. These batch timings differ from the explanation workload in Table~\ref{tab:faithfulness}.

\subsection{Feature Selection}
\label{sec:synthetic_ground_truth}

In real-world datasets, we rarely know which features generate each label. Deletion metrics therefore cannot establish whether an explanation identifies these features. Feature selection and support recovery have long been used for interpretable prediction~\citeauthorp{guyon2003introduction,tibshirani1996lasso,zou2005elastic,shaham2022deep}. Recent work also considers self-supervised and group-structured unsupervised feature selection~\citeauthorp{segal2024spectral,lifshitz2026groupfs}, as well as sparse linear support recovery using randomized, refined, or projected-gate estimators~\citeauthorp{lindenbaum2021rawls,lindenbaum2022refined,jana2023projected}. We use four controlled binary tasks adapted from LSPIN~\citeauthorp{yang2022locally}, where the generating features are known for each query and vary across examples. E1--E3 use nonlinear gated rules with query-dependent support. E4 has two latent generating groups and 46 nuisance features. We test whether explanations rank the generating features highly for each query. In E4, we evaluate recovery of the support for the query's latent group. However, all four informative variables are shifted by group, so variables outside this support can still help predict the label when the group is unobserved.

We compare \oursglobal, using one backbone estimator and the primary reference bank, with EBM, logistic regression, and MAPLE~\citeauthorp{plumb2018maple}. For each query, we rank features by the absolute values of the nonzero global-plus-residual coefficients for \ours, local coefficients for MAPLE, additive scores for EBM, and mean-relative standardized contributions for logistic regression. Each method selects up to the known number of generating features, giving all methods an oracle ranking budget. We report support F1 and the accuracy of each interpretable method's own predictions. We evaluate post-hoc explainers of \tabiclvtwo{} against the same supports and with the same budget, but measure accuracy after replacing unselected features with their means. Appendix~\ref{app:lspin_tasks} provides the full setup, and Appendix~\ref{app:regional_explanation_results} reports regional diagnostics.

Across the four tasks, seeds 21--30, and 100 queries per task and seed, \oursglobal achieves mean support F1 of $0.539\pm0.010$ and accuracy of $0.706\pm0.011$. Standard errors are computed across seeds after averaging over tasks within each seed. Among the interpretable-by-design predictors, \oursglobal has the highest average support F1 and performs best on E1--E3. MAPLE has higher support F1 on E4 ($0.624$ versus $0.548$). The post-hoc explainers of \tabiclvtwo{} are shown for reference rather than as direct competitors: their scores do not determine any prediction, whereas each interpretable-by-design method is evaluated on the rule it uses to predict.

\begin{table}[t]
  \centering
  \caption{Generating-support F1 and accuracy on LSPIN E1--E4 (10 seeds; 100 queries per task--seed). Blocks distinguish own-predictor from masked-\tabiclvtwo{} accuracy; E4 uses group-conditional support. An all-zero coefficient vector selects no features and receives support F1 zero.}
  \label{tab:lspin_selection}
  \scriptsize
  \resizebox{\linewidth}{!}{%
  \begin{tabular}{lcccccccc}
    \toprule
    & \multicolumn{2}{c}{\textbf{E1}} & \multicolumn{2}{c}{\textbf{E2}} & \multicolumn{2}{c}{\textbf{E3}} & \multicolumn{2}{c}{\textbf{E4}} \\
    \cmidrule(lr){2-3}\cmidrule(lr){4-5}\cmidrule(lr){6-7}\cmidrule(lr){8-9}
    \textbf{Method} & \textbf{F1}$\uparrow$ & \textbf{Acc.}$\uparrow$ & \textbf{F1}$\uparrow$ & \textbf{Acc.}$\uparrow$ & \textbf{F1}$\uparrow$ & \textbf{Acc.}$\uparrow$ & \textbf{F1}$\uparrow$ & \textbf{Acc.}$\uparrow$ \\
    \midrule
    \multicolumn{9}{l}{\textit{Interpretable-by-design predictors (own deployed accuracy)}} \\[2pt]
    \rowcolor{ourshl}
    \oursglobal & $0.512$ & $\underline{0.663}$ & $0.540$ & $\underline{0.675}$ & $0.554$ & $\mathbf{0.787}$ & $0.548$ & $0.697$ \\
    MAPLE & $0.408$ & $\mathbf{0.665}$ & $0.480$ & $\mathbf{0.684}$ & $0.455$ & $0.760$ & $0.624$ & $\mathbf{0.712}$ \\
    EBM & $0.439$ & $0.636$ & $0.361$ & $0.653$ & $0.494$ & $0.690$ & $0.541$ & $0.688$ \\
    Logistic (ERM) & $0.422$ & $0.535$ & $0.343$ & $0.649$ & $0.491$ & $0.656$ & $0.493$ & $0.632$ \\
    \midrule
    \multicolumn{9}{l}{\textit{Post-hoc explainers of \tabiclvtwo{} (masked-\tabiclvtwo{} accuracy)}} \\[2pt]
    LIME & $0.517$ & $0.645$ & $\mathbf{0.640}$ & $0.674$ & $\mathbf{0.713}$ & $0.743$ & $\mathbf{0.964}$ & $\underline{0.703}$ \\
    KernelSHAP & $0.525$ & $0.638$ & $0.616$ & $0.656$ & $\underline{0.689}$ & $\underline{0.768}$ & $0.556$ & $0.690$ \\
    Integrated Gradients & $\underline{0.533}$ & $0.639$ & $\underline{0.624}$ & $\mathbf{0.684}$ & $0.687$ & $0.764$ & $\underline{0.765}$ & $0.702$ \\
    GradientSHAP & $\mathbf{0.540}$ & $0.634$ & $0.617$ & $0.664$ & $0.662$ & $0.766$ & $0.581$ & $0.679$ \\
    \bottomrule
  \end{tabular}}
\end{table}

\subsection{Interpretability Evaluation}
\label{sec:faithfulness}

For each query, we sum the centered contributions of processed features that belong to the same original table column, and rank the columns by their signed contribution to the predicted-class score. This allows us to check whether removing the features that most support the prediction consistently decreases its probability. We use the same coefficients as the deployed predictor to obtain this ranking; Appendix~\ref{app:faithfulness_details} describes the exact mapping to the original columns.

Table~\ref{tab:faithfulness} uses 50 fixed queries from each of 10 CC-18 datasets (see Appendices~\ref{app:explanation_benchmark} and~\ref{app:explanation_protocol}). For both \oursglobal and the post-hoc methods, we evaluate explanations of the same complete \oursglobal predictor. We keep the predicted class fixed while progressively replacing features with their training-set means. Deletion Spearman measures the correlation between the number of features replaced and the cumulative decrease in the probability of that class, so it measures how consistently the probability decreases, rather than how much it decreases. Gen-SVM measures how well the selected features allow an SVM to reproduce the corresponding wrapper's predicted class, while Gen-kmeans measures how well clusters obtained from the selected features agree with the true labels. Neighbor stability measures how much the normalized attribution or selection scores change between nearby observed queries; lower values indicate greater stability. We report the regional results and the different sets of queries for which their metrics are valid in Appendix~\ref{app:regional_explanation_results}. Appendix~\ref{app:metric_definitions} gives the formal definitions.

\begin{table*}[t]
  \centering
  \caption{Explanation quality and online prediction-plus-explanation cost on 10 CC-18 datasets. We evaluate quality using 500 queries; all rows use the same 422 queries with valid values for deletion and 493 for the other metrics. We exclude undefined values rather than setting them to zero. We average runtime across datasets, giving each dataset equal weight and excluding one-off setup. For LIME and KernelSHAP attribution time, we use a fixed subset of 100 queries (10 per dataset, three repeats); for the common predictor latency and the remaining rows, we use 50 queries per dataset.}
  \label{tab:faithfulness}
  \scriptsize
  \setlength{\tabcolsep}{3.5pt}
  \resizebox{\linewidth}{!}{%
  \begin{tabular}{lccccc}
    \toprule
    \textbf{Method} & \textbf{Deletion Spearman}$\uparrow$ & \textbf{Gen-SVM}$\uparrow$ & \textbf{Gen-kmeans}$\uparrow$ & \textbf{Neighbour stab.}$\downarrow$ & \textbf{Online ms/query}$\downarrow$ \\
    \midrule
    \rowcolor{ourshl}
    \oursglobal & $\underline{0.470}$ & $\mathbf{0.886}$ & $\mathbf{0.697}$ & $0.032$ & $\mathbf{26.974}$ \\
    KernelSHAP & $0.384$ & $0.708$ & $0.555$ & $0.037$ & $53{,}686.935$ \\
    LIME & $0.449$ & $0.813$ & $0.601$ & $0.032$ & $14{,}607.496$ \\
    FastSHAP & $0.408$ & $0.608$ & $0.637$ & $0.025$ & $27.197$ \\
    IG on surrogate & $0.417$ & $\underline{0.861}$ & $\underline{0.669}$ & $0.031$ & $27.482$ \\
    GradientSHAP on surrogate & $0.446$ & $0.856$ & $0.666$ & $0.031$ & $27.560$ \\
    Input $\times$ Gradient on surrogate & $0.454$ & $0.830$ & $0.568$ & $0.034$ & $\underline{27.112}$ \\
    INVASE & $0.458$ & $0.514$ & $0.495$ & $\underline{0.007}$ & $27.191$ \\
    L2X & $0.441$ & $0.503$ & $0.536$ & $0.010$ & $27.213$ \\
    REAL-X & $\mathbf{0.534}$ & $0.500$ & $0.476$ & $\mathbf{0.004}$ & $27.192$ \\
    \bottomrule
  \end{tabular}%
  }
\end{table*}

\oursglobal has the highest Gen-SVM agreement and Gen-kmeans score, and the second-highest deletion Spearman correlation. Its neighbor stability is similar to LIME and the gradient-based methods, although the learned selectors are more stable. The learned post-hoc explainers add little online latency after the common predictor call, but they require 41.1--105.0 s per dataset for one-off construction. For 50 queries, excluding the predictor setup shared by all methods, \oursglobal takes 1.35 s, while the fastest post-hoc alternative takes 42.44 s (see Appendix~\ref{app:explanation_runtime_accounting}).

\paragraph{Component and configuration ablations.}
On the ten explanation-benchmark datasets, adding a local correction increases the shared model's accuracy from 76.13\% to 89.90\%. However, the global-only model has higher deletion Spearman (0.610 versus 0.530) and Gen-SVM agreement (95.4\% versus 88.8\%). This means that a model can score higher on these metrics even when its accuracy is lower, so we need to consider them together with accuracy.

When compared with the local-only model using the same $k=400$, the combined model has slightly lower accuracy (89.90\% versus 90.18\%), and we don't find a clear difference in deletion monotonicity in the paired comparison. However, the combined model gives a nonzero explanation for every query, while the local-only model does so for only 68\% of queries. On this subset, the main benefit we observe from adding the shared model is that the explanation is nonzero for every query.

Using $k=400$ rather than $k=30$ improves paired deletion Spearman by $0.130$ (95\% interval $[0.027,0.234]$), while $k=30$ gives 1.29 percentage points higher accuracy on this subset. On CC-18, the fast configuration ($k=30$, 30 local iterations) also gives 0.52 percentage points higher accuracy and is $4.6\times$ faster at prediction (Table~\ref{tab:runtime}). Reducing the local iteration cap from 300 to 30 shows no clear difference. Appendix~\ref{app:component_ablation} reports all nine configurations and their paired comparisons.

\subsection{Regression}

To test whether the method can also be used for continuous targets, we use a scalar prediction from the backbone as the teacher target. We build the reference bank using five-fold cross-fitting with shuffled K-fold splits instead of stratified folds. We standardize all teacher targets using the target mean and scale learned by the full-context backbone, so that the fixed regularization strength is comparable across datasets. \oursglobal fits its shared and local squared-loss ElasticNets to these standardized targets. For \oursreg, we use the same 64-anchor configuration as in classification, fit each anchor model to the standardized out-of-fold teacher targets, and combine the eight closest anchor models. Both methods transform their predictions back to the original target scale, without clipping or probability normalization. Table~\ref{tab:regression} reports predictive performance on the 119 TALENT regression datasets; we don't evaluate explanation quality here.

\begin{table}[t]
  \centering
  \caption{TALENT regression on 119 splits. nRMSE uses the test-target standard deviation; $\pm$ reports across-dataset variation.}
  \label{tab:regression}
  \scriptsize
  \setlength{\tabcolsep}{4pt}
  \begin{tabular}{llcc}
    \toprule
    \textbf{Family}
      & \textbf{Method}
      & \textbf{Mean nRMSE}$\downarrow$
      & \textbf{Mean $R^2$}$\uparrow$ \\
    \midrule
    \multirow{3}{*}{\shortstack[l]{Tabular foundation\\models}}
      & \tabiclvtwo
      & $\mathbf{0.479 \pm 0.300}$
      & $\mathbf{0.681}$ \\
      & \tabiclvtwo{} (1 estimator)
      & $\underline{0.482 \pm 0.303}$
      & $\underline{0.676}$ \\
      & TabPFNv2
      & $0.506 \pm 0.310$
      & $0.648$ \\
    \midrule
    \multirow{4}{*}{\shortstack[l]{Interpretable\\by design}}
      & \cellcolor{ourshl}\oursglobal
      & \cellcolor{ourshl}$0.515 \pm 0.287$
      & \cellcolor{ourshl}$0.653$ \\
      & \cellcolor{ourshl}\oursreg
      & \cellcolor{ourshl}$0.521 \pm 0.275$
      & \cellcolor{ourshl}$0.654$ \\
      & MAPLE
      & $0.571 \pm 0.298$
      & $0.586$ \\
      & EBM
      & $0.601 \pm 0.273$
      & $0.564$ \\
    \midrule
    \multirow{5}{*}{\shortstack[l]{Classical\\baselines}}
      & CatBoost
      & $0.537 \pm 0.297$
      & $0.623$ \\
      & Random Forest
      & $0.550 \pm 0.301$
      & $0.608$ \\
      & LightGBM
      & $0.566 \pm 0.333$
      & $0.569$ \\
      & XGBoost
      & $0.587 \pm 0.394$
      & $0.501$ \\
      & Ridge Regression
      & $0.702 \pm 0.272$
      & $0.434$ \\
    \bottomrule
  \end{tabular}
\end{table}

\oursglobal has an nRMSE of 0.515, trailing the \tabiclvtwo{} backbone model (0.482 nRMSE for the one-estimator case), but improving over the strongest tree baseline (0.537 nRMSE) and the other evaluated interpretable predictors. \oursreg{} is also competitive (0.521 nRMSE), with a slightly higher mean $R^2$ score (0.654 vs. 0.653 for \oursglobal{}). We leave further tuning of these approaches for regression tasks to future work.

\subsection{Limitations and Future Work}

\oursglobal{} is limited by the expressiveness of the linear model, and may struggle to capture highly interactive local residual functions in the retrieved neighborhood, particularly for higher-dimensional and multi-class problems. Our approach does not directly model categorical features, and we ordinally encode any categorical columns before treating them as numerical. As a result, our feature contributions are for the entire encoded column, rather than for specific categories. Handling categorical features directly and providing attributions for groups of categorical features could improve the interpretability of \oursglobal{} in this setting. Similarly, we do not currently evaluate our approach on text input. We leave evaluating \oursglobal{} with direct categorical handling and on text inputs, and augmenting the approach with richer local interpretable model components, to future work.

While our provided feature contributions exactly decompose the query-specific linear score, our explanations do not decompose the final normalized probability score or account for how features influence the learned representation, retrieval, and the fitted coefficients. Sparsity in the learned coefficients does not imply that the entire query-adaptive model only depends on the features with nonzero coefficients. Our retrieval approach (top-$k$) can also introduce discontinuities into our explanations, which can shift substantially based on small changes in the query point; the active set for ElasticNet may also vary. Further, our regional anchor model approach often does not provide a feature-level explanation, as the coefficients for the blended model (corresponding to the predicted class) are all zeros for 73.3\% of queries on the synthetic data and 43.8\% on real data.

Our real data interpretability study is limited to a single backbone model and a fixed subset of 10 datasets in a capped low-data regime. The studied metrics do not establish causal feature relevance; notably, Gen-SVM can utilize both the locations and values of retained features. Additionally, the gradient-based rows utilize a differentiable surrogate with unquantified approximation error. We also do not evaluate the interpretability of our regression approach. Our methods also require access to the internal representations of the tabular in-context learning backbone, and do not apply directly to fully black-box models.

\section{Conclusion}

We propose \oursglobal{}, an interpretable-by-design predictor leveraging a tabular foundation model to guide a linear predictor. \oursglobal{} learns a single ElasticNet model to capture global feature effects, and leverages the tabular foundation model to retrieve a similar set of samples for each query, learning a second ElasticNet to capture the residuals of the first. Together, these produce a query-specific linear prediction rule. \oursglobal{} provides feature contributions and a baseline that exactly sum to the linear prediction score for each class. We theoretically bound \oursglobal{}'s local approximation error, under a smoothness assumption on the residual in the learned representation space that includes an additive term for differences between labeled contexts.

Across four classification benchmarks, \oursglobal{} is competitive with the \tabiclvtwo{} backbone (within 0.3 to 1.6 percentage points in accuracy) and outperforms the strongest tree-based baseline on each. Across a synthetic benchmark for feature selection, \oursglobal{} is the top-performing interpretable predictor among those evaluated on average in support F1 score. On the real-data explanation benchmark, \oursglobal{} scores highest among the evaluated methods on two metrics measuring the utility of the selected features, and second on a metric measuring deletion monotonicity. We also propose \oursreg{}, which combines pre-fitted regional linear models; \oursreg{} is slower than a comparable single-estimator backbone model, but faster than the corresponding eight-estimator backbone model.

Our approach is limited by the fact that our explanations are only for the linear class scores, and by the expressiveness of local linear models. We leave exploring richer local interpretable models, leveraging additional tabular foundation models, and evaluating our approach's explanation quality for regression to future work.

\subsection*{Reproducibility statement}

We specify the details of our predictor and its primary hyperparameters in Section~\ref{sec:method} and the specifics of our benchmark selection, data splits, seeds, and row capping in Section~\ref{sec:experimental_setup}. In Appendix~\ref{app:preprocessing_mapping}, we detail the data preprocessing and raw feature mapping. In Appendix~\ref{app:predictive_baseline_configs}, we detail the configuration of our predictive baselines. In Appendix~\ref{app:reference_bank}, we detail the construction of our cross-fitted reference bank. In Appendix~\ref{app:elasticnet_solve}, we detail our ElasticNet solver. In Appendix~\ref{app:prediction_runtime}, we detail our runtime protocol and the configuration of our regional anchors. In Appendix~\ref{app:explanation_benchmark}, we detail our explanation benchmark; in Appendix~\ref{app:explanation_protocol}, we detail our explanation baselines; and in Appendix~\ref{app:metric_definitions}, we provide metric definitions. In Appendix~\ref{app:lspin_tasks}, we detail our synthetic task generation. In Appendix~\ref{app:residual_approx_proof}, we provide a proof of Theorem~\ref{thm:residual_approx}.

\subsection*{AI use statement for this paper}

We used generative AI to write, debug, and run code for this paper, including code for the implementation of our method, synthetic task generators, preprocessing for our explanation benchmark tasks, and code to tabulate results. We note that all experiments in this paper were designed by the authors and not by generative AI. The generative AI tools were not used to generate proof ideas, hypotheses, or conceptual frameworks, to design experiments, or to translate text. The generative AI tools were used to assist with literature searches, structure the paper, generate a draft, and edit the paper (i.e., change wording or the placement of author-written text, such as in proofs). They also helped check whether descriptions of theoretical claims and experimental results were consistent with the proofs and reported numbers. All code written by the generative AI was reviewed and tested by the authors to ensure it performed as expected. All numbers reported in this paper were checked to make sure they match the corresponding recorded output from the experiments. All proofs were checked by hand by the authors. All edits to the text were approved by the authors. The authors take full responsibility for this work, including all text, claims, and artifacts generated with the help of AI.
\clearpage
\bibliographystyle{unsrtnat}
\bibliography{bibfile_arxiv}

\appendix
\input{app_arxiv}
\end{document}

%% file: app_arxiv.tex
\clearpage

\section{Broader Impacts}
\label{app:broader_impacts}

\paragraph{Positive impacts.}
Interpretable machine learning is especially valuable in high-stakes decision domains such as healthcare, credit scoring, and legal risk assessment, where practitioners need to understand and audit individual predictions. \ours lowers the barrier to inspecting tabular predictions by producing additive per-feature contributions directly, without a separate post-hoc explainer pipeline. Such decompositions may make it easier for domain experts and regulators to audit model behaviour, provided that their empirical fidelity and stability are validated for the deployed configuration.

\paragraph{Negative impacts.}
Like other prediction and explanation tools, the method may be misused in high-stakes settings or lend unwarranted confidence to automated decisions. Practitioners may also over-interpret coefficients as causal evidence rather than predictive associations. We frame them explicitly as contributions to the deployed query-specific linear score; they do not separately decompose representation-based retrieval, coefficient fitting, or the clipping-and-normalization step, and they do not imply causality.

\section{Reference-Bank Construction}
\label{app:reference_bank}

The reference bank is constructed once for each training split using five-fold
stratified cross-fitting. We partition the training rows into five folds. For
each fold, we fit the backbone on the other four folds, present the held-out
rows as unlabeled queries, and record their class probabilities and
representations at the input to the first ICL decoder block. Concatenating the
held-out outputs gives every training row one out-of-fold probability target and one
query-role representation. This prevents a row's target and representation
from being produced by a backbone that was fitted on that row's label, and it
places reference rows in the same unlabeled-query role used at test time. A
separate backbone fitted to the complete training split supplies test-query
representations. The folds are used only to construct the reference bank; they
do not create an ensemble of deployed predictors.
The processed design vectors used by both ElasticNet stages are constructed
with the complete-training preprocessing map. The retained-column mask,
processed-coordinate order, and raw-column origin map are then held fixed;
fold-specific backbones provide out-of-fold targets and representations, not
alternative linear-design coordinates. Thus a stored row and a test query
enter the linear rule in the same processed feature space even though their
representations are obtained under different labeled contexts.
Cross-fitting prevents direct label leakage and places each reference row in an
unlabeled-query role, but it does not guarantee that representations produced
under different fold contexts are directly comparable to one another or to a
test representation produced under the complete context. Theorem~\ref{thm:residual_approx}
therefore states the required compatibility condition on the actual
context-specific outputs and includes an additive term for the resulting
cross-context shift.

\section{Weighted ElasticNet Optimization}
\label{app:elasticnet_solve}

Both stages of \ours use an unpenalized intercept and the same ElasticNet
penalty. Consider a generic problem with $M$ observations, processed feature
vectors $\mathbf z_i\in\mathbb R^D$, target vectors
$\mathbf t_i\in\mathbb R^C$, and normalized weights $w_i\geq0$ satisfying
$\sum_{i=1}^M w_i=1$. Define the weighted means
$\overline{\mathbf z}_w=\sum_{i=1}^M w_i\mathbf z_i$ and
$\overline{\mathbf t}_w=\sum_{i=1}^M w_i\mathbf t_i$. Let
$\mathbf Z_w\in\mathbb R^{M\times D}$ and
$\mathbf T_w\in\mathbb R^{M\times C}$ have rows
$\sqrt{w_i}(\mathbf z_i-\overline{\mathbf z}_w)^\top$ and
$\sqrt{w_i}(\mathbf t_i-\overline{\mathbf t}_w)^\top$, respectively. The
centered objective is
\[
  \min_{\mathbf A\in\mathbb R^{D\times C}}\;
  \frac{1}{2}\|\mathbf Z_w\mathbf A-\mathbf T_w\|_F^2
  +\alpha\rho\|\mathbf A\|_1
  +\frac{\alpha(1-\rho)}{2}\|\mathbf A\|_F^2,
  \qquad
  \mathbf a=\overline{\mathbf t}_w-\mathbf A^\top\overline{\mathbf z}_w.
\]
The recovered unpenalized intercept is $\mathbf a\in\mathbb R^C$.
The implementation jointly optimizes all class columns using FISTA with
backtracking and a Karush--Kuhn--Tucker residual check. It uses at most 300
iterations and tolerance $10^{-4}$. The shared model uses uniform weights and
the probability targets in the primary reference bank; the query-specific correction uses the normalized
similarity weights from Eq.~\eqref{eq:residual_neighbor_weights} and residual
targets $\mathbf r_i$. Both use $\alpha=0.01$ and $\rho=0.25$, corresponding
to an $\ell_1$ coefficient $\alpha\rho=0.0025$ and ridge weight
$\alpha(1-\rho)=0.0075$; because the ridge term is written with a factor of
$1/2$, the multiplier on $\|\mathbf A\|_F^2$ is $0.00375$.

\section{Proof of the Residual Approximation Bound}
\label{app:residual_approx_proof}

\begin{proof}[Proof of Theorem~\ref{thm:residual_approx}]
For brevity, write $\mathbf D=\mathbf D^*$, $\mathbf d=\mathbf d^*$,
$\widetilde w_i=\widetilde w_i^*$, and
$\overline{\mathbf r}_k=\sum_{i\in\mathcal N^*}\widetilde w_i\mathbf r_i$.
Because the intercept in Eq.~\eqref{eq:local_residual_elasticnet} is
unpenalized, its first-order optimality condition gives
$\mathbf d=\overline{\mathbf r}_k-\mathbf D^\top\overline{\mathbf z}_k$.
The shared term cancels from the approximation error, so
\begin{align*}
 &\|\boldsymbol\eta^*-
       \widehat{\boldsymbol\eta}_{\rm lin}^*\|_2 \\
 &\quad=\|\mathbf r^*-
       (\mathbf D^\top\mathbf z^*+\mathbf d)\|_2 \\
 &\quad\leq
   \|\mathbf r^*-\overline{\mathbf r}_k\|_2
   +\|\overline{\mathbf r}_k-
         (\mathbf D^\top\mathbf z^*+\mathbf d)\|_2.
\end{align*}

For the first term, the assumption on the retrieved residuals implies
\begin{align*}
 \|\mathbf r^*-\overline{\mathbf r}_k\|_2
 &\leq \sum_{i\in\mathcal N^*}\widetilde w_i
  \|\mathbf r^*-\mathbf r_i\|_2 \\
 &\leq L_r\Delta_k(\mathbf x^*)+\epsilon_{\rm ctx}.
\end{align*}
For the second term, the intercept identity yields
\[
 \|\overline{\mathbf r}_k-(\mathbf D^\top\mathbf z^*+\mathbf d)\|_2
 \leq \|\mathbf z^*-\overline{\mathbf z}_k\|_2\|\mathbf D\|_F.
\]

It remains to bound $\|\mathbf D\|_F$. By optimality, the objective at
$(\mathbf D,\mathbf d)$ is no larger than at the feasible local-constant
choice $(\mathbf 0,\overline{\mathbf r}_k)$. Dropping the nonnegative data-fit
and $\ell_1$ terms from the optimum gives
\begin{align*}
 \frac{\alpha(1-\rho)}{2}\|\mathbf D\|_F^2
 &\leq \frac{1}{2}\sum_{i\in\mathcal N^*}\widetilde w_i
       \|\mathbf r_i-\overline{\mathbf r}_k\|_2^2 \\
 &\leq \frac{1}{2}\sum_{i\in\mathcal N^*}\widetilde w_i
       \|\mathbf r_i-\mathbf r^*\|_2^2 \\
 &\leq \frac{\bigl(L_r\Delta_k(\mathbf x^*)+\epsilon_{\rm ctx}\bigr)^2}{2}.
\end{align*}
The middle inequality holds because the weighted mean
$\overline{\mathbf r}_k$ minimizes weighted squared distance over constant
vectors. Hence
$\|\mathbf D\|_F\leq
\bigl(L_r\Delta_k(\mathbf x^*)+\epsilon_{\rm ctx}\bigr)/\sqrt{\alpha(1-\rho)}$.
Substitution into the two error terms above proves
Eq.~\eqref{eq:residual_approx_bound}.
\end{proof}

The theorem states its assumption on the actual context-specific residuals,
and the proof uses it only for pairs formed by a retrieved row and the query;
no condition is imposed among the reference rows. Without the additive term
$\epsilon_{\rm ctx}$, any nonzero context-induced shift between a reference
row and the query at vanishing representation distance would force $L_r$ to
be arbitrarily large. The first error term alone is the error of the
local-constant correction $(\mathbf 0,\overline{\mathbf r}_k)$, which is
therefore at most $L_r\Delta_k(\mathbf x^*)+\epsilon_{\rm ctx}$; the result
bounds how much the regularized slope can add to this error rather than
establishing a first-order improvement from fitting it. It is tighter than an
analogous target-smoothness result only when the shared model reduces local
variation. The bound applies to the linear score, not to the subsequent
clipping and normalization. No claim is made that cross-fitting bounds
$\epsilon_{\rm ctx}$ or that the residual smoothness condition holds
uniformly on every dataset.

\section{Additional Experimental Details}
\label{app:experiments}

\subsection{Input Preprocessing and Raw-Feature Mapping}
\label{app:preprocessing_mapping}

Each benchmark loader preserves a fixed raw-column order and converts
categorical columns to ordinal integer codes. For CC-18, codes are created when
the full OpenML table is downloaded (a missing category becomes NaN); for
TabArena-Lite, the encoder is fitted to the full table before the seeded split.
TALENT fits the encoder on its supplied training split and applies it to
validation and test, while BeyondArena fits it on each provided training split;
both assign unseen categories code $-1$. After the benchmark's training block
is formed, remaining NaNs are replaced by training-column means; for TALENT
this block is the concatenated supplied training and validation data. No
one-hot, target, or interaction expansion is applied. Labels are integer
encoded independently of the feature preprocessing.

For the single-estimator TICL models, the resulting numerical matrix is passed
through the TabICL~2.0.3 preprocessing pipeline fitted on the training block.
The map removes constant columns, standardizes each retained column and clips
the standardized value to $[-100,100]$, applies the Yeo--Johnson power
transform with standardization, and applies TabICL's two-stage soft outlier
clipping with threshold four. The same fitted transformations are applied to
reference rows and test queries. With one estimator, the recorded first-member
feature permutation fixes the processed-coordinate order. Consequently, every
retained processed coordinate has exactly one originating raw column; a raw
constant column has no retained coordinate and receives score zero.

Let $g(j)$ denote the raw-column origin of processed coordinate $j$. For class
$c$, the displayed raw-column contribution is
\[
 A_{r,c}(\mathbf x^*)=
 \sum_{j:g(j)=r}(z_j^*-\overline z_j)W^*_{j,c},
 \qquad \mathbf W^*=\mathbf P+\mathbf D^*.
\]
Thus grouping uses a \emph{sum of contributions}, which preserves the exact
linear-score decomposition. In the reported experiments the sum contains at
most one term because preprocessing does not expand a column. Synthetic
support recovery instead attaches each transformed-space combined coefficient
$W^*_{j,c}$ to its raw origin before taking absolute values and ranking. It
does not sum coefficients, rescale them into raw measurement units, or use the
centered contribution formula. This distinction prevents the raw-feature
display operation from changing the coefficient-ranking estimand.

\subsection{Predictive-Baseline Configurations}
\label{app:predictive_baseline_configs}

Table~\ref{tab:predictive_baseline_configs} records the fixed configurations
used for Table~\ref{tab:main}. Runs used Python~3.10 with TabICL~2.0.3,
TabPFN~2.2.1, scikit-learn~1.6.1, CatBoost~1.2.10, XGBoost~3.1.0,
LightGBM~4.6.0, InterpretML~0.7.8, AutoGluon~1.5.0, pytabkit~1.7.3,
TabArena~0.0.1, and PyTorch~2.9.1. None of the methods received a
benchmark-wide hyperparameter search; validation below means within-fit model
selection or early stopping, never selection on the reported test split.

\begin{table}[H]
  \centering
  \caption{Configurations for the predictive baselines in Table~\ref{tab:main}.
  Suite order for ensemble counts is CC-18/TabArena-Lite/TALENT/BeyondArena.
  ``Defaults'' means the named package's released default configuration except
  for the explicitly listed settings.}
  \label{tab:predictive_baseline_configs}
  \scriptsize
  \setlength{\tabcolsep}{3pt}
  \begin{tabular}{p{0.17\linewidth}p{0.48\linewidth}p{0.27\linewidth}}
    \toprule
    \textbf{Method} & \textbf{Fixed configuration} & \textbf{Selection/validation} \\
    \midrule
    \tabiclvtwo & TabICL; 8/4/4/4 estimators; packaged defaults otherwise. & No tuning or validation selection. \\
    TabPFNv2 & TabPFN; 8/8/4/8 estimators; native multiclass model for at most 10 classes and one-vs-rest above 10; packaged defaults otherwise. & No tuning or validation selection. \\
    KernelICL & Repository implementation with four TabICL estimators; $k$NN uses $k=20$ and Gaussian uses scale 4. & Fixed a priori; no calibration or search. \\
    MAPLE & scikit-learn random forest with 200 trees, \texttt{max\_features}=0.5, minimum leaf size 10; per-query ridge penalty $0.001$. & Stratified 80/20 training/validation split; retained count chosen by validation RMSE from candidates up to $\min(D,100)$, using at most 50 validation queries. \\
    EBM & InterpretML; 128 bins, 3000 maximum rounds, learning rate 0.01, and no interactions. & Fixed; package-internal fitting only, with no external search. \\
    NAM & One hidden layer of width 32 per feature; at most 300 epochs; batch size 256; AdamW with learning rate $10^{-3}$ and weight decay $10^{-4}$. & Stratified 80/20 training/validation split; validation-loss early stopping with patience 30. \\
    TabR & pytabkit \texttt{TabR\_S\_D\_Classifier} defaults. & No external search; package training procedure. \\
    RealMLP, ModernNCA, FT-Transformer & AutoGluon RealMLP and FT-Transformer, and the TabArena ModernNCA wrapper, using released defaults. FT-Transformer worker counts were set to zero. & Fixed stratified 80/20 training/validation split passed to each package; no configuration search. \\
    CatBoost & 300 iterations, depth 6, learning rate 0.1. & Fixed; no validation set or early stopping. \\
    Random Forest & scikit-learn; 300 trees, remaining defaults. & Fixed; no validation set. \\
    LightGBM/XGBoost & 300 trees, maximum depth 6, learning rate 0.1; histogram XGBoost. & Fixed; no validation set or early stopping. \\
    Logistic regression & scikit-learn; $C=1$, at most 1000 iterations; SAGA for CC-18, TabArena-Lite, and BeyondArena, and L-BFGS for TALENT. & Fixed; no validation set. \\
    \bottomrule
  \end{tabular}
\end{table}

FT-Transformer used its released 300-feature guard. On BeyondArena tasks above
that limit, a deterministic training-variance screen retained the 300
highest-variance features. On CC-18 tasks above that limit, the guard was
instead disabled and training used batch size eight; this was an execution
setting, not test-set tuning.

\subsection{Compute Resources}
\label{app:compute}

Experiments were distributed across NVIDIA GPUs on university compute servers. The fit/prediction measurements in Table~\ref{tab:runtime} were collected on NVIDIA L4 GPUs, while the explanation measurements in Table~\ref{tab:faithfulness}, synthetic evaluation in Table~\ref{tab:lspin_selection}, and regression evaluation in Table~\ref{tab:regression} were collected on NVIDIA A100 80\,GB GPUs. Timing protocols use device synchronization; parallel benchmark execution was used only to reduce wall-clock time and is not treated as a latency measurement.

\subsection{Dataset Licences}
\label{app:licenses}

All benchmark datasets are publicly available. The CC-18 benchmark datasets are sourced from OpenML~\citeauthorp{bischl2021openml}; individual dataset licenses vary by uploader, ranging from CC0 (public domain) to CC~BY~4.0. The \tabiclvtwo codebase~\citeauthorp{qu2026tabiclv2} is used under the BSD~3-Clause License. All other software dependencies (scikit-learn, NumPy, etc.) are used under their respective open-source licenses (BSD/MIT).

\subsection{Prediction Runtime}
\label{app:prediction_runtime}

Table~\ref{tab:runtime} compares three TICL runtime configurations.
\tabiclvtwo and TabPFNv2 use their default ensemble sizes of eight
and four estimators, respectively. All TICL variants use one
\tabiclvtwo estimator and a batched five-fold out-of-fold reference bank.
Both global-residual variants use $\alpha=0.01$ and $\rho=0.25$
at both stages, with 300 iterations for the shared fit.
The standard configuration uses $k=400$ and 300 local iterations;
TICL-Fast uses $k=30$ and 30 local iterations. In Table~\ref{tab:runtime}, it reduces prediction time from 6.849 to 1.495~s and increases accuracy by 0.52 percentage points relative to the standard global-residual configuration. Separate ablations show an accuracy--deletion-monotonicity trade-off for smaller neighbourhoods (Appendix~\ref{app:component_ablation}).
Regional settings are given below.
These runtime configurations use fixed-step FISTA without the
backtracking used by the primary global model.

Timings were measured on NVIDIA L4 GPUs with one worker per GPU and one CPU math-library thread. CUDA synchronization brackets each timed operation, and measurements are averaged over three repetitions. Fit time includes model construction, deterministic preprocessing, backbone fitting, construction of the prediction-and-representation bank, and fitting of the shared ElasticNet or regional anchor models. Prediction time covers the full held-out batch, including query representations, neighbour retrieval and residual ElasticNet fitting for the global variant, or anchor selection and coefficient blending for the regional variant, followed by class-score clipping, probability normalization, and contribution construction. Metric evaluation is excluded. Seeds are averaged within each dataset before datasets are weighted equally. For reference, the one-estimator \tabiclvtwo operating point attains $0.8710$ accuracy with $0.485$~s fit and $0.745$~s prediction time; the fast and regional variants therefore predict faster than the eight-estimator row in Table~\ref{tab:runtime}, but not than its single-estimator counterpart.

\paragraph{Regional anchor configuration.}
\label{app:regional_anchors}
The regional variants in Tables~\ref{tab:main}, \ref{tab:runtime}, and
\ref{tab:regression}, and the diagnostics below, use up to 64 anchors.
We apply $k$-means to normalized reference-bank representations,
select the training row closest to each cluster centre,
and remove duplicate selections.
Each classification anchor model fits the backbone probability targets on its
200 nearest reference rows using weighted squared loss and
ElasticNet with $\alpha=0.05$ and $\rho=0.5$.
Neighbour weights use a softmax over cosine similarities with
temperature $0.1$.
For each query, the eight most similar anchors receive softmax
blending weights with temperature $0.05$.
Their coefficients and intercepts are combined before computing
class probabilities and feature contributions.
For regression, the same construction fits standardized scalar out-of-fold
teacher predictions and omits probability clipping and normalization.
Anchor fitting uses at most 300 iterations in Tables~\ref{tab:main} and
\ref{tab:regression} and in the diagnostics below,
and 50 in Table~\ref{tab:runtime}. The runtime measurements use the same L4
protocol described above.

\paragraph{Regional explanation diagnostics.}
\label{app:regional_explanation_results}
We treat the anchor model as a computational variant rather than a coequal
explanation method because its fitted bank often produces an all-zero
predicted-class contribution vector. On the synthetic tasks, its support
F1/accuracy pairs are $0.040/0.659$, $0.069/0.689$, $0.067/0.784$, and
$0.204/0.708$ on E1--E4, respectively. Its mean support F1 is
$0.095\pm0.008$ and its mean accuracy is $0.710\pm0.009$; only $26.7\%$ of
queries have nonzero coefficient vectors ($14.5\%$, $20.3\%$, $22.9\%$, and
$49.0\%$ on E1--E4), with a median of one selected feature when nonempty.
Predictions can nevertheless vary through query-dependent routing and blended
intercepts.

On the real-data benchmark, the regional model obtains deletion Spearman
$0.381$, Gen-SVM $0.715$, Gen-kmeans $0.632$, neighbour stability $0.044$,
and 1.828~ms/query online latency. Only 281 of 500 queries have nonzero
contribution vectors; deletion and stability cover eight datasets, while
generalization covers seven. These values use the regional model's own valid
population and are therefore descriptive rather than directly comparable to
Table~\ref{tab:faithfulness}.

\subsection{Explanation Benchmark Setup}
\label{app:explanation_benchmark}

The explanation benchmark used in the main paper is a fixed 10-dataset subset selected before score-function tuning. Each dataset uses a stratified 80/20 split with seed 42, with the training and test sets capped at 1600 and 400 rows, and 50 explained test queries per dataset drawn with query seed 123 (500 queries in total). The subset, split, and query indices were then held fixed for all reported comparisons. The datasets are: \texttt{balance-scale}, \texttt{banknote-authentication}, \texttt{har}, \texttt{mnist\_784}, \texttt{kr-vs-kp}, \texttt{vowel}, \texttt{eucalyptus}, \texttt{texture}, \texttt{letter}, and \texttt{jungle\_chess\_2pcs\_raw\_endgame\_complete}.

\subsection{Explanation Baselines and Deletion Protocol}
\label{app:explanation_protocol}

The post-hoc explanation methods and \oursglobal target the same complete \oursglobal wrapper; \oursreg targets its own deployed regional wrapper. They are evaluated on the same ordered query blocks, with the same masking operator and the same fill value for a removed feature: the training-set mean of that raw feature on the corresponding split. The class predicted by the relevant unperturbed wrapper is frozen and used as the explained class for class-specific methods. Features are ranked once per query and removed cumulatively in that order to a depth of $\min(10,D_{\rm raw})$, and we record the drop in the frozen-class probability after each step. During a perturbation, the original ordered query block is fixed and only the focal row is replaced. Prediction calls rerun the full relevant wrapper because the query representation and, for \oursglobal, the neighborhood and residual correction may change. This is a deployment-oriented benchmark rather than a test of causal relevance.

For LIME, we use \texttt{LimeTabularExplainer} on the original training table with \texttt{discretize\_continuous=False}, \texttt{sample\_around\_instance=True}, and 500 perturbation samples per explained query. For SHAP, we use KernelSHAP with a random background subset of 32 training examples and 64 evaluation samples per explained query, seeded independently per query. The global-residual wrapper is not differentiable end-to-end because it contains top-$k$ retrieval and a per-query ElasticNet fit. Integrated Gradients (24 steps), GradientSHAP (24 samples), and Input$\times$Gradient are therefore computed on a differentiable surrogate, a two-hidden-layer MLP (width 256) fitted to the wrapper's class probabilities on the training rows only. The reported gradient rows consequently conflate attribution behavior with surrogate approximation error; held-out surrogate fidelity was not separately measured. FastSHAP and experimental adaptations of INVASE, L2X, and REAL-X are trained on the training rows with the wrapper's class probabilities as targets. These settings are kept fixed across datasets.

For \oursglobal, processed-feature contributions are aggregated to raw semantic columns using the preprocessing map. \oursglobal, LIME, KernelSHAP, FastSHAP, and the surrogate-gradient methods return a signed score per raw feature for the frozen class, which we rank in descending order. INVASE, L2X, and REAL-X instead return nonnegative, class-agnostic selector scores in these adaptations; we rank those scores in descending order. Their rows therefore evaluate selected-feature rankings, not signed class contributions.

\subsection{Interpretability Metric Definitions}
\label{app:metric_definitions}

Let $\mathcal Q$ index the queries in the 10-dataset interpretability
benchmark. For $q\in\mathcal Q$, let
$\mathbf x_q\in\mathbb R^{D_{\rm raw}}$ be the raw query,
$\widehat y_q\in\{1,\ldots,C\}$ its frozen predicted class, and
$\boldsymbol\psi:\mathbb R^{D_{\rm raw}}\rightarrow
\mathbb R^{D_{\rm raw}}$ the raw-feature attribution map for class
$\widehat y_q$ produced by a class-specific method, or the selector-score map
for a selector method. We abbreviate its value
at the query as $\boldsymbol\psi^{(q)}:=\boldsymbol\psi(\mathbf x_q)$ and let
$\pi_q$ be the permutation of $\{1,\ldots,D_{\rm raw}\}$ that orders features
by decreasing method-specific score $\psi_j^{(q)}$, with deterministic tie breaking.
Let $f_c(\mathbf x)\in[0,1]$ denote the probability that the deployed wrapper
for the evaluated row assigns to class $c$ (\oursglobal for the non-regional
rows and \oursreg for its own row), and let
$\mu_j=|\mathcal T|^{-1}\sum_{\mathbf x\in\mathcal T}x_j$ be the mean of raw
feature $j$ over the corresponding training split $\mathcal T$. Every metric
is first averaged over the valid queries of a dataset and then averaged with
equal weight across the contributing datasets. Undefined values are excluded rather than
set to zero. The non-regional Gen-SVM, Gen-kmeans, and neighbour-stability
entries use the 493-query intersection for which every included method returns
a finite, nonzero attribution vector. Their deletion entry additionally
requires a defined deletion correlation, leaving 422 common-valid queries
across nine datasets. The regional row instead uses its own defined population:
281 queries across eight datasets for deletion and neighbour stability. Its
generalization entries have 556 valid calibration rows and 281 valid evaluation
rows, with both metrics defined on seven datasets. Consequently, the regional
quality entries are not on the same query population as the non-regional rows.

\paragraph{Deletion monotonicity.} Let $M=\min(10,D_{\rm raw})$ and let $\mathbf{x}_q^{(m)}$ denote the query with its first $m$ features under ranking $\pi_q$ replaced by their training-set means. Define the cumulative probability drop
\[
  \Delta_m^{(q)} \;=\; f_{\widehat y_q}(\mathbf{x}_q) - f_{\widehat y_q}(\mathbf{x}_q^{(m)}),
  \qquad m = 1, \ldots, M.
\]
Deletion Spearman is the per-query correlation between $m$ and $\Delta_m^{(q)}$. Higher values indicate that progressive removal of highly ranked features produces a more consistently increasing decrease in the predicted-class probability. It measures monotonicity rather than the magnitude of the decrease: an arbitrarily small but consistently increasing drop can score highly. Deletion-based evaluation is used broadly in prior work~\citeauthorp{petsiuk2018rise,hooker2019roar,deyoung2020eraser}; these citations do not imply that each work uses this exact Spearman statistic.

\paragraph{Gen-SVM and Gen-kmeans.} Following the local feature-selection evaluation of LSPIN~\citeauthorp{yang2022locally}, we use the budget $K=\min(10,\max(1,D_{\rm raw}-1))$. Let $J_q^{(K)}$ be the set containing the first $K$ indices under $\pi_q$, and let $\mathbf{x}_q^{J_q^{(K)}}$ denote the masked input that retains $\mathbf{x}_q$ on $J_q^{(K)}$ and replaces every other entry $j$ with $\mu_j$. \emph{Gen-SVM} trains a linear SVM ($C=1$) on a calibration block of 100 training rows, each masked by its own explanation and labeled with the wrapper's predicted class. It reports agreement with $\widehat y_q$ from the masked evaluation queries. \emph{Gen-kmeans} runs $k$-means with $C$ clusters on the masked evaluation queries and reports the accuracy of the Hungarian-matched cluster assignment against the true labels. The two metrics therefore target different labels and are reported separately. Moreover, because every row is masked according to its own selected set, Gen-SVM can exploit the pattern of retained locations as well as retained values; we did not run a mask-only control.

\paragraph{Neighbour stability.} Neighbour stability compares normalized attribution differences between nearby observed queries. Write $\mathbf u=\boldsymbol\psi/\|\boldsymbol\psi\|_1$. Let $\widetilde{\mathbf h}_q$ be the retrieval representation of query $q$, standardized with the training-row mean and standard deviation, and let $\mathcal N_3(q)$ contain the three nearest other evaluation queries in Euclidean distance. Then
\[
  L_{\text{nb}}(\mathbf x_q)=\max_{q'\in\mathcal N_3(q)}
  \frac{\|\mathbf u(\mathbf x_q)-\mathbf u(\mathbf x_{q'})\|_2}
       {\|\widetilde{\mathbf h}_q-\widetilde{\mathbf h}_{q'}\|_2},
\]
and lower values indicate more similar explanations for nearby queries.

\subsection{Feature-Importance Scoring}
\label{app:faithfulness_details}

For class $c$, the processed-feature contribution is
$a_{j,c}(\mathbf{x}^*)=(z_j^*-\overline z_j)W^*_{j,c}$ with
$\mathbf W^*=\mathbf P+\mathbf D^*$. The raw-column contribution sums these
terms over the common origin map $g$ defined in
Appendix~\ref{app:preprocessing_mapping}. Unless a protocol fixes $c$ in
advance, we use the class predicted for the unperturbed query. Raw
contributions are ranked by signed value in the real-data benchmark
(Appendix~\ref{app:explanation_protocol}); the synthetic support-recovery
tasks instead rank the absolute transformed-space combined coefficients after
associating them with their raw origins (Appendix~\ref{app:lspin_tasks}).

For the synthetic-task notation below,
$\sigma(t):=(1+\exp(-t))^{-1}$ is the logistic function,
$\mathcal N(\mu,v)$ denotes a normal distribution with mean $\mu$ and variance
$v$, $q$ indexes a test query, and $S_q$ is its designated generating-support
set of raw-feature indices.
\subsection{LSPIN Synthetic Tasks: E1--E4}\label{app:lspin_tasks} Tasks E1--E4 are regime-dependent binary classification benchmarks adapted from the LSPIN synthetic evaluation~\citeauthorp{yang2022locally}. Each task is generated with ten independent seeds, 21--30. Unless noted otherwise, all features are drawn from independent standard normal distributions, and labels follow $y \sim \mathrm{Bernoulli}(\sigma(\ell))$ where $\sigma$ is the logistic function and $\ell$ is the task-specific logit defined below. \par \paragraph{E1--E3: regime-switching with a gate feature (11 features).} Each task uses $n = 2000$ samples with features $x_1, \ldots, x_{11} \overset{\text{i.i.d.}}{\sim} \mathcal{N}(0, 1)$. The gate feature $x_{11}$ determines which rule is active for each instance: the \emph{gate branch} applies when $x_{11} < 0$ and the \emph{off-gate branch} applies otherwise. Because $x_{11}$ governs branch selection, it is included in the local support $S_q$ of every query. Training uses $1800$ samples and testing uses $200$ samples; $100$ test queries are sampled per seed. \par \noindent\textbf{E1.} Gate branch ($x_{11} < 0$): $\ell = x_1 x_2$; \quad support $S_q = \{x_1,\, x_2,\, x_{11}\}$ ($|S_q|=3$). Off-gate branch ($x_{11} \geq 0$): $\ell = x_3^2 + x_4^2 + x_5^2 + x_6^2 - 4$; \quad support $S_q = \{x_3,\, x_4,\, x_5,\, x_6,\, x_{11}\}$ ($|S_q|=5$). \par \noindent\textbf{E2.} Gate branch ($x_{11} < 0$): $\ell = x_1 x_2$; \quad support $S_q = \{x_1,\, x_2,\, x_{11}\}$ ($|S_q|=3$). Off-gate branch ($x_{11} \geq 0$): $\ell = -10\sin(0.2\,x_7) + |x_8| + x_9 + e^{-x_{10}} - 2.4$; \quad support $S_q = \{x_7,\, x_8,\, x_9,\, x_{10},\, x_{11}\}$ ($|S_q|=5$). \par \noindent\textbf{E3.} Gate branch ($x_{11} < 0$): $\ell = x_3^2 + x_4^2 + x_5^2 + x_6^2 - 4$; \quad support $S_q = \{x_3,\, x_4,\, x_5,\, x_6,\, x_{11}\}$ ($|S_q|=5$). Off-gate branch ($x_{11} \geq 0$): $\ell = -10\sin(0.2\,x_7) + |x_8| + x_9 + e^{-x_{10}} - 2.4$; \quad support $S_q = \{x_7,\, x_8,\, x_9,\, x_{10},\, x_{11}\}$ ($|S_q|=5$). \par \paragraph{E4: two-group design with 46 nuisance features (50 features total).} E4 uses $n = 1000$ samples drawn from two equally-sized groups. Features $x_5, \ldots, x_{50}$ are independent nuisance variables with $x_j \sim \mathcal{N}(0, 0.5)$ (variance $0.5$). The first four features form the informative core, with group-shifted distributions that make group membership recoverable from the data: \par \noindent\textbf{Group A} ($n = 500$): $x_1, x_2, x_3, x_4 \overset{\text{i.i.d.}}{\sim} \mathcal{N}(+1,\, 0.5)$. Logit: $\ell = x_1 x_2 - 0.9$. Generating-group conditional support: $S_q = \{x_1,\, x_2\}$ ($|S_q|=2$). \par \noindent\textbf{Group B} ($n = 500$): $x_1, x_2, x_3, x_4 \overset{\text{i.i.d.}}{\sim} \mathcal{N}(-1,\, 0.5)$. Logit: $\ell = x_3^2 + x_4^2 - 2.5$. Generating-group conditional support: $S_q = \{x_3,\, x_4\}$ ($|S_q|=2$). \par Support follows the latent generating group. Because all four core variables shift with group, variables outside this conditional support can reveal group membership and therefore carry information about the observable $P(y\mid\mathbf x)$; $S_q$ should not be read as its unambiguous support. We randomly split each group into 450 training and 50 test rows, yielding 900 training rows and 100 balanced test queries. \par \paragraph{Evaluation protocol.}
All tasks use one \tabiclvtwo estimator, one representation view, the primary reference bank, and squared loss. The \oursglobal row uses $k=400$, $\tau=0.1$, $\alpha=0.01$, and $\rho=0.25$; \oursreg uses the anchor configuration in Appendix~\ref{app:regional_anchors}. For query $q$, let $A_q$ contain up to $|S_q|$ raw-feature indices with the largest nonzero absolute predicted-class coefficients after associating each processed coefficient with its raw-column origin. The reported metric is conventional set F1,
\[
 \operatorname{F1}_q=\frac{2|A_q\cap S_q|}{|A_q|+|S_q|},\qquad
 \operatorname{F1}_q=0\ \text{when }A_q=\varnothing.
\]
For post-hoc explainers, the selected features are used to compute masked-\tabiclvtwo{} accuracy. Results average queries and seeds per task. Aggregate standard errors use ten seed-level averages over the four tasks.

\oursglobal ranks nonzero absolute combined raw coefficients; MAPLE ranks absolute local coefficients; EBM ranks absolute additive scores (no interactions); and logistic regression ranks absolute mean-relative standardized contributions. LIME, KernelSHAP, Integrated Gradients, and GradientSHAP use the absolute scores returned for the teacher's predicted class. Interpretable models report their own accuracy; post-hoc methods report \tabiclvtwo accuracy after mean-masking unselected features. MAPLE selects its feature count on a validation split within the training data.

\paragraph{Gate-feature interpretation.}
The gate $x_{11}$ belongs to every E1--E3 generating support, but may have a small combined contribution when a neighborhood lies within one regime. Generating support and local contribution magnitude therefore need not coincide.

\section{Additional Background and Related Work}
\label{app:related_work}

We organize prior work along three axes that are central to this paper: whether the predictor is \emph{interpretable by design} or requires post-hoc analysis, whether the predictive form exposes explicit feature contributions, and whether predictions adapt locally to each query. We distinguish post-hoc explanations, which analyze a fixed predictor after inference, from interpretability by design, where the prediction itself is computed through an inspectable structure. Table~\ref{tab:method_positioning} summarizes representative methods under this view and positions \ours as an interpretable-by-design, global-plus-local linear predictor.

\paragraph{Tabular in-context learning models.}
Recent work has introduced a new paradigm for tabular prediction in which a labeled dataset is provided as context at inference time. Given a labeled context $\{(\mathbf{x}_i, y_i)\}_{i=1}^N$ with raw input $\mathbf{x}_i \in \mathbb{R}^{D_{\rm raw}}$ and class label $y_i \in \{1, \ldots, C\}$, and a query input $\mathbf{x}^* \in \mathbb{R}^{D_{\rm raw}}$, these models directly predict
\[
p_\theta\!\left(y^* \mid \mathbf{x}^*, \{(\mathbf{x}_i, y_i)\}_{i=1}^N\right)
\]
in a single forward pass, without dataset-specific training. Here $y^*$ is the
unknown query label, $\theta$ denotes the pretrained model parameters, and
$p_\theta(\cdot\mid\cdot)$ is the resulting conditional class distribution.
This corresponds to in-context learning, in which adaptation to a new task is
performed by conditioning on labeled examples rather than by gradient updates.

During pretraining, models are exposed to a distribution of datasets and optimized to predict labels for held-out examples given their context, encouraging a general-purpose inference procedure. Architecturally, different designs vary in how tabular data is represented and how interactions are modeled. TabPFN~\citeauthorp{hollmann2023tabpfn} represents rows as tokens and performs attention across samples, while TabPFNv2~\citeauthorp{hollmann2025tabpfnv2} introduces a cell-based design with alternating row and column attention, increasing flexibility at a higher computational cost. \tabicl~\citeauthorp{qu2025tabicl} improves scalability through a two-stage design that first constructs row embeddings and then performs attention over them. Throughout this work, we use \tabiclvtwo~\citeauthorp{qu2026tabiclv2}, the successor of \tabicl, as our backbone and primary baseline.

These models are strong predictors, but their decision-making process is not interpretable by design: predictions are computed via multi-layer attention rather than through explicit feature-level operations. In Table~\ref{tab:method_positioning}, they therefore lie outside the interpretable-by-design, linear, and local region. Rather than attempting to explain these models post hoc, \ours reuses their context-conditioned representation geometry to define locality, while constructing a separate predictor whose computation is interpretable by design.

\paragraph{Post-hoc local explanation methods.}
Post-hoc explanation methods approximate a trained predictor after it has produced a prediction. Model-agnostic approaches such as LIME~\citeauthorp{ribeiro2016lime} and SHAP~\citeauthorp{lundberg2017shap} construct local surrogate models or estimate feature attributions by perturbing the input. These methods operate at the input-output level and do not expose the internal computation that produces the prediction. As a result, they require repeated evaluations of an external predictor and define locality independently of its internal representation.

Gradient-based attribution methods, including Integrated Gradients~\citeauthorp{sundararajan2017integrated}, DeepLIFT and Input $\times$ Gradient~\citeauthorp{shrikumar2017learning}, and GradientSHAP-style estimators~\citeauthorp{kokhlikyan2020captum}, instead use differentiability to assign feature importance scores. A separate line of work amortizes explanation through learned selector--predictor models that return instance-wise feature subsets, including L2X~\citeauthorp{chen2018l2x}, INVASE~\citeauthorp{yoon2019invase}, REAL-X~\citeauthorp{jethani2021realx}, and FastSHAP~\citeauthorp{jethani2022fastshap}; ILLUME~\citeauthorp{piaggesi2025illume} similarly learns meta-representations for local surrogate explanations. LIME, SHAP, and gradient attribution analyze a fixed predictor after inference. The original L2X, INVASE, and REAL-X formulations instead train selector and predictor components jointly or in coordinated objectives; our Table~\ref{tab:faithfulness} adaptations train their selectors against the fixed \oursglobal wrapper outputs. In each case, the selector or explainer score is distinct from the explicit rule used by that wrapper to make its prediction.

\paragraph{Interpretable-by-design models and local predictors.}
Interpretable-by-design models expose feature-level structure directly in the deployed predictor. Classic examples include sparse linear models, such as Least Absolute Shrinkage and Selection Operator (Lasso)~\citeauthorp{tibshirani1996lasso}, and additive models, such as Explainable Boosting Machines (EBMs) and Neural Additive Models (NAMs)~\citeauthorp{nori2019interpretml,agarwal2021nam}. These models are interpretable by design and often partially linear, but they are typically global: a single model is applied across all queries.

Local modeling approaches instead fit simple predictors in a neighborhood around each query. Methods such as LOESS~\citeauthorp{cleveland1988loess} and local polynomial regression~\citeauthorp{fan1996local} often use input-space distance. MAPLE~\citeauthorp{plumb2018maple} learns a supervised forest geometry and derives query-dependent proximity weights before fitting a local model. Its forest is fixed after training, but its neighborhood is label-informed and changes with the query. In contrast, \ours recomputes neighborhoods from representations conditioned on the current labeled ICL context; the distinction is the context-conditioned geometry, not an absence of query adaptation in MAPLE.

LSPIN and its locally linear variant LLSPIN~\citeauthorp{yang2022locally} obtain query adaptation differently. Building on stochastic-gate feature selection~\citeauthorp{yamada2020stg}, a gating network trained jointly with the prediction network maps each sample to gates that select a sparse subset of its features. When the prediction network is linear (LLSPIN), the deployed model is a sample-specific sparse linear rule whose coefficients are the shared weights multiplied by the sample's gates. Locality is thus learned from the labels by a dataset-specific gating network, whereas \ours obtains it from the representation of a pretrained in-context model and fits each local rule with a per-query ElasticNet. Our synthetic tasks and selected-feature utility metrics are adapted from this work (Section~\ref{sec:synthetic_ground_truth} and Appendix~\ref{app:metric_definitions}).

\ours differs from these prior directions in how locality is defined and reused. A shared ElasticNet model captures dataset-level structure, while context-conditioned representations from a pretrained tabular in-context learner determine which training rows fit each query's residual correction. The deployed predictor is the sum of this global coefficient matrix and a query-specific ElasticNet correction in the processed feature space. It is therefore interpretable by design and query-adaptive, with locality grounded in the foundation model's learned geometry rather than raw input distances.

\paragraph{Interpretable extensions of tabular in-context models.}
Recent work has explored incorporating explanations directly into tabular in-context models. ShapPFN~\citeauthorp{sena2026shappfn}, a TabPFN-style model pretrained on synthetic datasets, decodes a base term and per-feature, per-class contributions from the test-row embeddings and sums them to obtain its logits. Its prediction is therefore additive over features by construction, and predictions and contributions are produced in a single forward pass. ExplainerPFN~\citeauthorp{fonseca2026explainerpfn} instead targets model-free feature-importance estimation: it predicts Shapley-style attributions for a new dataset without access to the model being explained.

ShapPFN is interpretable by design in the additive sense of Table~\ref{tab:method_positioning}, because its displayed contributions sum to its output. Each contribution, however, is computed by attention layers over the full labeled context, so how a term depends on its feature value is not itself inspectable, and no explicit feature-level rule is deployed. In \ours, the contribution of processed feature $j$ to class $c$ is the product $(z_j^*-\overline z_j)W^*_{j,c}$ of the centered feature value and an explicit coefficient; the foundation model only supplies prediction targets and defines the neighborhood. ExplainerPFN, like other post-hoc methods, produces attributions that do not determine a deployed prediction.

\paragraph{Example-based explanations.}
Example-based methods explain predictions by identifying similar training instances, as in prototype and metric-based approaches~\citeauthorp{snell2017prototypical}. These methods provide sample-level evidence but do not produce explicit feature-level contributions. \ours retains the neighborhood view but uses it to adapt a shared linear predictor through a residual correction, yielding both sample-level locality and feature-level contributions within a single interpretable-by-design model.

\section{Explanation Runtime Accounting}
\label{app:explanation_runtime_accounting}

The runtime column of Table~\ref{tab:faithfulness} reports equal-dataset means.
The online boundary includes the complete predictor call and construction of
the attribution, but excludes one-off model and explainer setup.  TICL and the
seven learned or gradient explainers use the same 50 original queries per
dataset, measured with device synchronization over three repeats on two
NVIDIA A100 80\,GB GPUs.  LIME and KernelSHAP use the first 10 fixed original
queries from each of the same 10 datasets, also with three isolated,
device-synchronized repeats; their attribution times are added to the common
predictor latency measured on the 50-query blocks.  No stability-perturbed rows
enter this timing subset.  LIME retains 500 perturbations per explanation and
KernelSHAP retains 64 evaluation samples with a 32-row background. Extrapolating
the measured equal-dataset mean attribution times to 500 queries and three
repetitions gives approximately 28.5 A100-hours for LIME and KernelSHAP
combined, motivating the fixed 100-query timing subset.

The exact query-parallel \oursglobal implementation prepares the GPU query state once
per call and parallelizes only the independent residual ElasticNet solves.  It
uses the primary $k=400$, 300-iteration, $10^{-4}$-tolerance configuration.  On
all ten datasets and all repeats, its predicted classes matched the frozen
classes, and its probabilities, coefficients, and contributions had zero
measured difference across repeats.  Its equal-dataset mean joint prediction
and contribution-construction latency is 26.974~ms/query.  The corresponding
latency of \oursreg is 1.828~ms/query; all 500 outputs are finite, while 281
have a nonzero predicted-class contribution vector. The corresponding
end-to-end means are 27.197 for FastSHAP, 27.482 for Integrated Gradients,
27.560 for GradientSHAP, 27.112 for Input$\times$Gradient, 27.191 for INVASE,
27.213 for L2X, and 27.192~ms/query for REAL-X.  After adding the common
predictor call, the matched LIME and KernelSHAP measurements are 14,607.496
and 53,686.935~ms/query, respectively.  Across the three complete repeats, the
standard deviations of the equal-dataset means are 19.407 and 149.744~ms/query.

One-off method-specific construction is kept separate from online latency
because its per-query amortization depends on deployment volume.  The measured
equal-dataset mean construction times are 0.003~s for LIME, 1.024~s for
KernelSHAP, 58.366~s for FastSHAP, 104.961~s for
the shared differentiable surrogate used by Integrated Gradients,
GradientSHAP, and Input$\times$Gradient, 41.077~s for INVASE, 41.181~s for L2X,
and 71.341~s for REAL-X.  The gradient-surrogate values retain their cached
construction provenance; all other construction values are means of the three
new construction measurements.  The backbone, reference-bank, and
shared-ElasticNet fit is common to every method and is not charged only to
\oursglobal.

Panel~(c) of Figure~\ref{fig:selling-point} adds method-specific construction
to the 50-query online workload while excluding shared predictor construction.
This gives 1.349~s for \oursglobal and 42.436~s for
INVASE, the fastest post-hoc total, corresponding to a $31.5\times$ advantage.
The analogous LIME and KernelSHAP totals, extrapolated from their fixed
10-query-per-dataset timing subset, are 730.378 and 2,685.371~s, or
$542\times$ and $1,991\times$ the \oursglobal total.

\subsection{Component and Configuration Ablations}
\label{app:component_ablation}

We evaluate nine configurations on the ten explanation-benchmark
datasets, using seed 42, full-test accuracy and 50 fixed explanation
queries per dataset. All configurations share the same single-estimator
FP32 TabICLv2 teacher and batched five-fold OOF reference bank.
Unlike the ablation, the regional evaluation in
Appendix~\ref{app:regional_explanation_results} uses automatic mixed precision
and sequential five-fold OOF. The separate regional evaluation and this
ablation use the same 500 query identities but different fitted banks, yielding
281 and 277 nonzero regional explanations, respectively.
Explanation metrics follow Appendix~\ref{app:metric_definitions},
with each configuration explaining its own predictor.

The primary configuration combines a shared global ElasticNet with
a local residual correction using $k=400$ and at most 300 iterations.
Matched local-only fitting removes the global component while retaining
the local neighbourhood, regularization and iteration budget.
Global-only fitting removes the local correction.
The regional comparison holds the fitted anchor bank fixed and changes
only the blending rule.

\begin{table}[t]
  \centering
  \caption{Component and configuration ablations.
  Bold identifies the primary configuration, not column-wise best results.
  Iteration caps apply to the local correction; shared global fits retain
  300 iterations. A nonzero explanation has a finite contribution vector
  with maximum absolute contribution above $10^{-12}$.
  Explanation metrics use configuration-specific valid populations and
  equally weight contributing datasets.
  Gen-SVM targets each predictor's own classes; Gen-kmeans uses true labels.
  Evaluation counts appear in Table~\ref{tab:ablation_valid_counts}.}
  \label{tab:component_ablation}
  \small
  \setlength{\tabcolsep}{4pt}
  \resizebox{\linewidth}{!}{%
  \begin{tabular}{lrrrrr}
    \toprule
    Configuration &
    Acc.\ (\%)$\uparrow$ &
    Nonzero expl.\ (\%)$\uparrow$ &
    Deletion Spearman$\uparrow$ &
    Gen-SVM (\%)$\uparrow$ &
    Gen-kmeans (\%)$\uparrow$ \\
    \midrule
    \multicolumn{6}{l}{\textit{Global and local components}} \\
    Local ElasticNet only ($k=400$, matched)
      & 90.18 & 68.0 & 0.5207 & 70.72 & 55.77 \\
    Global ElasticNet only
      & 76.13 & 100.0 & 0.6097 & 95.40 & 67.80 \\
    \textbf{TICL --- primary ($k=400$, 300 iterations)}
      & \textbf{89.90} & \textbf{100.0} & \textbf{0.5297}
      & \textbf{88.80} & \textbf{71.40} \\
    \midrule
    \multicolumn{6}{l}{\textit{Local reference configuration}} \\
    Local ElasticNet reference ($k=100$)
      & 90.29 & 41.8 & 0.2807 & 53.03 & 52.97 \\
    \midrule
    \multicolumn{6}{l}{\textit{Correction neighbourhood and iteration cap}} \\
    TICL ($k=30$, 300 iterations)
      & 91.19 & 100.0 & 0.3993 & 81.60 & 69.40 \\
    TICL ($k=400$, 30 iterations)
      & 89.92 & 100.0 & 0.5323 & 89.20 & 72.20 \\
    TICL ($k=30$, 30 iterations)
      & 90.97 & 100.0 & 0.4194 & 82.60 & 69.60 \\
    \midrule
    \multicolumn{6}{l}{\textit{Regional blending}} \\
    Regional --- nearest anchor
      & 87.88 & 44.6 & 0.2867 & 61.27 & 52.06 \\
    Regional --- eight-anchor blend
      & 89.21 & 55.4 & 0.3975 & 71.69 & 61.83 \\
    \bottomrule
  \end{tabular}%
  }
\end{table}

\paragraph{Global and local components.}
Adding the local correction raises accuracy by 13.77 percentage points
over global-only fitting. Against matched local K400, the primary model
produces more nonzero explanations but has slightly lower accuracy.
Its paired Deletion Spearman difference is $0.027$, with a 95\%
dataset-bootstrap interval of $[-0.078,0.114]$, indicating no clear
difference in deletion monotonicity.
Higher reported Gen scores use different eligible populations and
should not be interpreted as matched comparisons.
Global-only fitting has higher deletion and Gen-SVM scores despite
lower accuracy, illustrating the need to assess explanation diagnostics
alongside predictive performance.
Local K100 is a reference setting rather than an isolated component
ablation because its regularization also differs.

\paragraph{Neighbourhood and iteration cap.}
At 300 iterations, K30 improves accuracy by 1.29 percentage points,
while K400 improves paired Deletion Spearman by
$0.130\,[0.027,0.234]$.
This supports an accuracy--deletion-monotonicity trade-off on this subset.
At K400, 30 iterations yield similar observed results:
the paired Spearman difference for 30 minus 300 iterations is
$-0.0003\,[-0.0104,0.0093]$.
The larger cap therefore has no demonstrated advantage in this
experiment; the interval does not establish equivalence.

\paragraph{Regional blending.}
Using eight anchors instead of the nearest anchor raises accuracy
from 87.88\% to 89.21\% and the fraction of nonzero explanations
from 44.6\% to 55.4\%.
Its paired Deletion Spearman difference is
$0.031\,[-0.116,0.157]$, so improved deletion monotonicity is not
established. This comparison isolates blending within the regional
model, not regionalization against the primary model.
These ablations do not measure runtime or establish a universally
optimal configuration.

\paragraph{Settings.}
Matched local K400 and all global models use $\alpha=0.01$ and
L1 ratio $0.25$; local K100 uses $\alpha=0.05$ and L1 ratio $0.5$.
Local and global models use safeguarded FISTA with tolerance $10^{-4}$;
neighbour weighting uses temperature $0.1$.
Shared global and local-only fits retain 300-iteration caps.
Regional variants share up to 64 anchors fitted on 200 neighbours
with FISTA, 300 iterations, $\alpha=0.05$ and L1 ratio $0.5$.
Eight-anchor blending uses temperature $0.05$.
These settings are distinct from runtime configurations using
different solvers or iteration budgets.

\paragraph{Valid populations and uncertainty.}
Empty explanations and undefined correlations are excluded rather
than assigned zero. The primary model's 12 undefined deletion
correlations have constant probability-drop curves; all its
explanations are nonzero.
The primary value 0.5297 above uses its 488 own-valid queries across all
ten datasets. Table~\ref{tab:faithfulness} instead reports 0.470 on the
422-query cross-method intersection over nine contributing datasets; on
its own valid population there, the primary score is 0.5283.
Paired comparisons use common-valid query indices, with each model's
original predicted class fixed separately.
Intervals use 10,000 dataset-bootstrap resamples; they are exploratory,
unadjusted for multiple comparisons and conditional on one split seed.

\begin{table}[t]
  \centering
  \caption{Evaluation populations for the ablation.
  Every configuration has 500 attempted explanation queries and
  full-test accuracy on all ten datasets.
  Dataset counts refer to deletion Spearman, Gen-SVM and Gen-kmeans,
  respectively.}
  \label{tab:ablation_valid_counts}
  \small
  \begin{tabular}{lrr}
    \toprule
    Configuration & Valid deletion /500 & Contributing datasets \\
    \midrule
    Matched local K400 & 336 & 9 / 9 / 9 \\
    Global-only & 500 & 10 / 10 / 10 \\
    TICL --- primary & 488 & 10 / 10 / 10 \\
    Local K100 & 209 & 7 / 6 / 6 \\
    TICL ($k=30$, 300 iterations) & 493 & 10 / 10 / 10 \\
    TICL ($k=400$, 30 iterations) & 488 & 10 / 10 / 10 \\
    TICL ($k=30$, 30 iterations) & 491 & 10 / 10 / 10 \\
    Regional --- nearest anchor & 222 & 7 / 5 / 6 \\
    Regional --- eight-anchor blend & 277 & 8 / 7 / 7 \\
    \bottomrule
  \end{tabular}
\end{table}